\documentclass[twoside,11pt]{article}

\usepackage{jmlr2eNoEditor}
\usepackage{tikz}
\usepackage{amsmath}
\usepackage{cleveref}
\usepackage{algorithm}
\usepackage{algorithmic}

\newcommand{\bi}{\begin{itemize}}
\newcommand{\ei}{\end{itemize}}
\newcommand{\bal}{\begin{align}}
\newcommand{\eal}{\end{align}}

\newcommand{\EE}{\mathbb{E}}
\newcommand{\PP}{{P}}
\newcommand{\RR}{\mathbb{R}}
\newcommand{\QQ}{{Q}}

\newcommand{\bI}{I}

\newcommand{\bD}{\mathbf{D}}

\newcommand{\bK}{K}

\newcommand{\cX}{\mathcal{X}}
\newcommand{\cY}{\mathcal{Y}}
\newcommand{\cZ}{\mathcal{Z}}

\newcommand{\cN}{\mathcal{N}}

\newcommand{\cP}{\mathcal{P}}

\newcommand{\cS}{\mathcal{S}}

\newcommand{\cG}{\mathcal{G}}
\newcommand{\cF}{\mathcal{F}}

\newcommand{\eps}{\epsilon}

\newcommand{\tsig}{\tilde{\sigma}}

\newcommand{\hY}{\hat{Y}}

\DeclareMathOperator{\Tr}{Tr}

\newcommand{\grad}{\bigtriangledown}

\newcommand{\ciG}{\mathring{G}}
\newcommand{\ciH}{\mathring{H}}
\DeclareMathOperator{\argmin}{argmin}
\DeclareMathOperator{\Ima}{Im}

\def\<{\langle}
\def\>{\rangle}
\newcommand{\diag}{{\rm{diag}}}
\renewcommand{\l}{\left}
\renewcommand{\r}{\right}

\ShortHeadings{Understanding Entropic Regularization in GANs}{D. Reshetova, Y. Bai, X. Wu and A. \"Ozg\"ur}
\firstpageno{1}

\begin{document}

\title{Understanding Entropic Regularization in GANs}

\author{\name Daria Reshetova 
    	\email resh@stanford.edu\\
        \addr Department of Electrical Engineering\\
        Stanford University\\
        Stanford,  CA 94305, USA
        \AND
		Yikun Bai
		\email bai@udel.edu\\
		\addr Department of Electrical and Computer Engineering\\
		University of Delaware\\
		Newark, DE 19716, USA
		\AND
		Xiugang Wu
		\email xwu@udel.edu\\
		\addr Department of Electrical and Computer Engineering\\
		University of Delaware\\
		Newark, DE 19716, USA
		\AND
		Ayfer \"Ozg\"ur
		\email aozgur@stanford.edu\\
        \addr Department of Electrical Engineering\\
        Stanford University\\
        Stanford,  CA 94305, USA
}


\maketitle

\begin{abstract}
Generative Adversarial Networks (GANs) are a popular method for learning distributions from data by modeling the target distribution as a function of a known distribution. The function, often referred to as the generator, is optimized to minimize a chosen distance measure between the generated and  target distributions. One commonly used measure for this purpose is the Wasserstein distance. However, Wasserstein distance is hard to compute and optimize, and in practice entropic regularization techniques are used to facilitate its computation and improve numerical convergence. The influence of regularization on the learned solution, however, remains not well-understood. In this paper, we study how several popular entropic regularizations of Wasserstein distance impact the solution learned by a Wasserstein GAN in a simple benchmark setting where the generator is linear and the target distribution is high-dimensional Gaussian. We show that entropy regularization of Wasserstein distance promotes sparsification of the solution, while replacing the Wasserstein distance with the Sinkhorn divergence recovers the unregularized solution. The significant benefit of both regularization techniques is that they remove the curse of dimensionality suffered by Wasserstein distance. We show that in both cases the optimal generator can be learned to accuracy $\epsilon$ with $O(1/\epsilon^2)$ samples from the target distribution without requiring to constrain the discriminator. We thus conclude that these regularization techniques can improve the quality of the generator learned from empirical data in a way that is applicable for a large class of distributions.
\end{abstract}


\section{Introduction}\label{sec:intro}
{G}{enerative} Adversarial Networks (GANs) have become a popular framework for learning data distributions and sampling as they have achieved impressive results in various domains, including image super resolution~\citep{photo}, image-to-image translation~\citep{image}, text to image synthesis \citep{text} and analyzing social networks \citep{video}. 
	As opposed to traditional methods of fitting a parametric distribution, GANs' objective is to find a mapping from a known distribution to the unknown data distribution or its empirical approximation. The mapping is set to a minimizer of a chosen distance measure between the generated and target distribution.
	
	
	In the original GAN framework, the distance measure is the Jensen-Shannon divergence \citep{GANs}. This measure was later replaced by the Wasserstein distance in \citep{arjovsky2017wasserstein}, and the follow-up works showed that Wasserstein GANs can help resolve several issues related to the original formulation, such as the lack of continuity,  mode collapse  \citep{arjovsky2017wasserstein} and vanishing gradients~\citep{gulrajani2017improved}.
	
	Despite these advantages, minimizing the Wasserstein distance between the target (data) and the generated distribution is a computationally challenging task. Indeed, computing the Wasserstein distance between two empirical distributions involves the resolution of a linear program whose cost can quickly become prohibitive whenever the size of the support of these measures or the number of samples exceeds several hundreds. A popular approach to facilitate the computation of the Wasserstein distance is to regularize it with an entropic term which makes the problem strongly convex and hence solvable by matrix scaling
	algorithms \citep{cuturi2013sinkhorn,balaji2019entropic}. More recent results have shown that this also results  in faster convergence and stability of the first-order methods used for optimizing Wasserstein GANs \citep{sanjabi2018convergence}.
	
	However, the impact of these regularization methods on the generator learned by the Wasserstein GAN remains poorly understood. This is partly due to the fact that GANs are primarily evaluated on real data, typically images, and although clearly valuable, such evaluations are often subjective due to lack of clear baselines for benchmarking. In this paper, we follow the philosophy advocated in \citep{feizi2017understanding} and focus on a simple benchmark setting where solutions can be explicitly characterized and compared. Following  \citep{feizi2017understanding}, we assume that the generator is linear and the target distribution is high-dimensional Gaussian. The population solution for the Wasserstein GAN in this setup has been characterized in \citep{feizi2017understanding}, where it was further shown that even in this simple setting the learning problem suffers from the curse of dimesionality -- the empirical solution learned on $n$ samples of the target distribution converges to the population solution  as $\Omega(n^{-2/d}),$ where $d$ is the dimension of the target distribution support. To resolve this sample complexity issue,  \citep{feizi2017understanding} then proposes to restrict the discriminator to be quadratic. This insight is arguably based on knowing that the sought target distribution is Gaussian, in which case the optimal discriminator is indeed quadratic and this restriction does not impact the optimal generator.  However, this insight does not generalize beyond the linear/Gaussian setting as for  non-Gaussian  data the  generator  obtained  under  a  quadratic  discriminator  is  not necessarily  the  one  minimizing  the  Wasserstein distance between the generated and the target distributions.
	
	In this paper, by focusing on the linear generator and Gaussian distribution setting of \citep{feizi2017understanding}, we explore how regularization impacts what generator is learnt and how it leads to better generalization. We study two slightly different ways of regularizing: entropic regularization \citep{cuturi2013sinkhorn} and Sinkhorn divergence \citep{genevay2018learning}. Extending our previous results \citep{reshetova2021understanding}, we show that the former introduces bias to the solution as if one were to constrain the nuclear norm of the covariance matrix of generator's output distribution, while Sinkhorn divergence results in the same solution as the unregularized Wasserstein GAN in \citep{feizi2017understanding}. We then show, in the more general case of sub-gaussian distributions and Lipschitz generators, that these regularizations result in sample complexity of $O_d(1/\sqrt{n}),$ thus overcoming the curse of dimensionality in \citep{feizi2017understanding} without explicitly constraining the discriminator. This indicates that adding regularization implicitly constrains the discriminator in a way suitable for a large class of distributions.
	\section{Preliminaries}\label{sec:sec-formulation}
	
	In this section, we provide some background  on optimal transport and optimal transport GANs. 
	
	%
	\subsection{Wasserstein GANs}
	Let $\cP(\cX)$ be the set of all probability measures with support $\cX$ and finite second moments. For $\PP_Z\in\cP(\cZ)$ and $\PP_Y\in\cP(\cY)$, denote by $\Pi(\PP_Z,\PP_Y)$ the set of all couplings of $\PP_Z$ and $\PP_Y,$ i.e. all joint probability measures from $\cP(\cZ\times\cY)$ with marginal distributions being $\PP_Z$ and $\PP_Y.$ 
	The squared Wasserstein distance between $P_Z, P_Y\in \cP(\mathbb R^d)$ under $\ell_2$ metric, or simply the \emph{squared 2-Wasserstein distance}, is defined as
	\begin{align}\label{eq:OT}
	W_{2}^2(\PP_Z,\PP_Y) =    \inf_{\pi\in\Pi(\PP_Z,\PP_Y)}\!\!\EE_{\pi}\l[\|Z-Y\|^2\r].
	\end{align} 
It can be verified that 2-Wasserstein distance is a metric between probability distributions in $\cP(\mathbb R^d)$; in particular, it is symmetric with respect to its two arguments, satisfies triangle inequality, and $W_{2}(\PP_Y,\PP_Y) = 0$.

The main objective of GANs is to find a mapping $G(\cdot)$, called generator, that comes from a set of functions ${\cG\subseteq\{G:\cX\to\cY\}}$ and maps a latent random variable $X\in \cX$ with some known distribution to a variable $Y\in\cY$ with some target probability measure $P_Y$. In the population case, we assume that we have access to $P_Y$, the true distribution of $Y$, while in the empirical case one has access to only a finite sample $\{Y_i\}_{i=1}^n$, hence the empirical distribution of $Y$. Using the squared 2-Wasserstein distance to measure the dissimilarity between the generated and target distribution leads to the following learning problem of GAN, referred to as \emph{W2GAN}:
	\begin{align}\label{eq:GAN}
	\min_{G\in\cG}W_{2}^2\l(\PP_{G(X)},\PP_Y\r).
	\end{align}


A remarkable feature of the Wasserstein distance is that strong duality holds for the minimization problem described in \eqref{eq:OT}, and hence the objective, squared 2-Wasserstein distance, in \eqref{eq:GAN} can be equivalently written in its dual form \citep[Theorem 5.10 and equation (5.12)]{villani2009optimal}:

\begin{align}
    W_{2}^2(\PP_{G(X)},\PP_Y)&=\sup_{\substack{\psi\in L^1(P_{G(X)}),\phi\in L^1(P_Y)\\\psi(G(x))-\phi(y)\leq \|G(x)-y\|^2}} \EE\left[\psi(G(X)) - \phi(Y))\right] \nonumber\\
    &=
    \sup_{\phi\in \text{Conv}(G(\cX))}\EE\left[\|G(X)\|^2 - 2\phi(G(X)) + \|Y\|^2 - 2\phi^*(Y))\right]\label{OT_dual}
\end{align}
where $\text{Conv}(\cZ)$ is the set of all (lower semicontinuous) convex functions on $\cZ$ and
 $L^1(P_Z)$ is the set of all functions whose absolute value has a finite measure: 
 $\phi \in L^1(P_Z)\iff \EE[|\phi(Z)|]<\infty$.

Note that the above optimization problem is maximizing a concave objective over a set of functions (i.e. discriminators), instead of optimizing over couplings as in the primal form \eqref{eq:OT}. This naturally leads to the min-max game formulation of GANs, where the generator seeks to generate samples that are close to the real data training samples, and it competes with a discriminator that seeks to distinguish between real and generated samples. 

The function $\phi$ can then be parametrized by a neural network resulting in the following architecture
\begin{center}
	\vspace{-3mm}
	\begin{tikzpicture}[thick,scale=0.8, every node/.style={scale=0.8}]
	\tikzset{box/.style={
			draw, rectangle, rounded corners=0.8pt,
			text width=2.2cm, text height=0.75cm}}
	\tikzset{node/.style={draw, circle, inner sep=0.5cm}}
	
	\path (0, 0) coordinate [node]
	(x) node {$X\sim P_X$};
	
	\path (2.5, 0) coordinate [box] (G) node {\begin{tabular}{c} Generator\\ $G:\cX\to\cY$ \end{tabular}};
	
	\path (5, -1) coordinate [node] (GX) node {$G({X})$};
	\path (5, +1) coordinate [node] (Y) node {${Y\sim P_Y}$};
	
	\path (7.5, -1) coordinate [box]
	(D_G) node {\begin{tabular}{c} Discriminator\\ $\phi:G(\cX)\to\RR$ \end{tabular}};
	\path (7.5, +1) coordinate [box]
	(D_r) node {\begin{tabular}{c} Discriminator\\ $\phi^*:\cY\to\RR$ \end{tabular}};
	\path (10, 0) coordinate [box,text width=0.75cm]
	(l) node {\begin{tabular}{c}
	     loss\\
	     $W^2_{2}$
	\end{tabular}};
	
	\draw [->] (Y) -- (D_r);
	\draw [->] (GX) -- (D_G);
	\draw [->] (x) -- (G);
	\draw [->] (D_r) -- (l);
	\draw [->] (D_G) -- (l);
	\draw [->] (G) edge[out=0] (GX);
	\end{tikzpicture}
\end{center}

\subsection{Entropic Wasserstein GANs}
	
In practice, the Wasserstein distance in \eqref{eq:OT} is often regularized to facilitate its computation leading to the  \emph{entropy regularized 2-Wasserstein distance}  \citep{cuturi2013sinkhorn}:
	\begin{align}
	W_{2,\lambda}^2(\PP_Z,\PP_Y)=\inf_{\pi\in \Pi(\PP_Z,\PP_Y)}\EE_{\pi}\l[\|Z-Y\|^2\r] 
	+\lambda I_{\pi}(Z;Y)\label{eq:OTreg}
	\end{align}
	where the regularization term is the mutual information $I_{\pi}(Z;Y)$  calculated according to the the joint distribution $\pi$. The corresponding \emph{entropic W2GAN} is defined as 
	\begin{align}\label{eq:EGAN}
	\min_{G\in\cG}W_{2,\lambda}^2\l(\PP_{G(X)},\PP_Y\r).
	\end{align}
While the entropic Wasserstein distance allows for faster computation, note that it can be strictly larger than zero even if the generated distribution is exactly the same as the target distribution, i.e. $W_{2,\lambda}^2(\PP_Y,\PP_Y)\neq 0$. This issue can be resolved by adding corrective terms to \eqref{eq:OTreg} \citep{genevay2018learning}, which leads to the Sinkhorn divergence:
	\begin{align}
	S_{\lambda}&(\PP_{G(X)},\PP_Y)=W_{2,\lambda}^2(\PP_{G(X)},\PP_Y)- \l(W_{2,\lambda}^2(\PP_{G(X)},\PP_{G(X)})+W_{2,\lambda}^2(\PP_Y,\PP_Y)\r)/2. \label{eq:OTsink}
	\end{align}
	One can easily check that 
	$S_{\lambda}(\PP_{Y},\PP_Y)=0$ for any $\PP_Y$. The corresponding \emph{Sinkhorn W2GAN} is given by:
	\begin{align}
	\min_{G\in\cG} S_{\lambda}(\PP_{G(X)},\PP_Y).
	\label{eq:OTsinkgan}
	\end{align}

Analogous to the case of the Wasserstein distance, the entropic Wasserstein distance also has a dual formulation which makes it  suitable for GAN optimization problems. This dual formulation does not involve optimizing over all couplings, but instead the search space is the set of all essentially bounded functions \citep{chizat2018scaling}:
\begin{align}
    W_{2,\lambda}^2(\PP_{G(X)},\PP_Y)&=\hspace{-1em}\sup_{\psi\in L_{\infty}(P_Y),\phi\in L_{\infty}(P_{G(\cX)})}\hspace{-1em} \EE\left[\psi(Y) + \phi(G(X))\right]+\lambda\nonumber\\
    &\quad-\lambda\EE_{(X,Y)\sim P_X\times P_Y}\l[e^{\frac{\phi(G(X))+\psi(Y)-\|G(X)-Y\|_2^2}{\lambda}}\r]\label{eq:OTreg_dual},
\end{align}
where $L_{\infty}(P_Y)$ is the set of all essentially bounded functions, i.e. $\phi \in L_{\infty}(P_Y)\iff \exists C>0:\; \PP\{\phi(Y)>C\} = 0.$

The so-called dual potentials $\phi,\psi$ can be parametrized by neural networks resulting in the following architecture
\begin{center}
	\vspace{-3mm}
	\begin{tikzpicture}[thick,scale=0.8, every node/.style={scale=0.8}]
	\tikzset{box/.style={
			draw, rectangle, rounded corners=0.8pt,
			text width=2.2cm, text height=0.75cm}}
	\tikzset{node/.style={draw, circle, inner sep=0.5cm}}
	
	\path (0, 0) coordinate [node]
	(x) node {$X\sim P_X$};
	
	\path (2.5, 0) coordinate [box] (G) node {\begin{tabular}{c} Generator\\ $G:\cX\to\cY$ \end{tabular}};
	
	\path (5, -1) coordinate [node] (GX) node {$G({X})$};
	\path (5, +1) coordinate [node] (Y) node {${Y\sim P_Y}$};
	
	\path (7.5, -1) coordinate [box]
	(D_G) node {\begin{tabular}{c} Discriminator\\ $\phi:G(\cX)\to\RR$ \end{tabular}};
	\path (7.5, +1) coordinate [box]
	(D_r) node {\begin{tabular}{c} Discriminator\\ $\psi:\cY\to\RR$ \end{tabular}};
	\path (10, 0) coordinate [box,text width=0.75cm]
	(l) node {\begin{tabular}{c}
	     loss\\
	     $W^2_{2,\lambda}$
	\end{tabular}};
	
	\draw [->] (Y) -- (D_r);
	\draw [->] (GX) -- (D_G);
	\draw [->] (x) -- (G);
	\draw [->] (D_r) -- (l);
	\draw [->] (D_G) -- (l);
	\draw [->] (G) edge[out=0] (GX);
	\end{tikzpicture}
\end{center}

Note that in \eqref{eq:OTreg_dual}, there are no constraints on the dual potentials, which makes the dual form suitable to implement with Neural networks, while 2-Wasserstein distance  requires convexity/quadratically bounded differences for the dual potential and 1-Wasserstein distance, another popular metric used in GANs, requires Lipschitz continuity of the dual potential. The constraints on the discriminators then give rise to various heuristics (\citep{korotin2019wasserstein,liu2019wasserstein} for 2-Wasserstein GANs and \citep{arjovsky2017wasserstein,wei2018improving} for 1-Wasserstein GANs) since the constraints cannot be handled exactly.

When one of the measures is an empirical distribution supported on $\{y_i\}_{i=1}^n$, which is often the case in GANs, only the values of $\psi$ on the empirical samples influence the solution, thus letting $\psi_i = \psi(y_i)$ and plugging in the empirical measure in place of $P_Y$ simplifies \eqref{eq:OTreg_dual} to 
\begin{align}\label{eq:OTreg_semi_dual}
    W_{2,\lambda}^2(\PP_{G(X)},\PP_Y)&=\sup_{\psi\in \RR^n,\phi\in L_{\infty}(G(\cX))} \sum_{i=1}^n\frac{\psi_i}{n} + \EE\left[\phi(G(X)
    )\right] +\lambda\\\nonumber
    &\quad-\lambda\EE\l[ \sum_{i=1}^ne^{\frac{\phi(G(X))+\psi_i-\|G(X)-y_i\|_2^2}{\lambda}}/n\r]
\end{align}
The above form is especially useful for optimization since one of the parametric functions becomes a vector. 

Given the optimal dual potentials, the optimal coupling can be found as \citep{janati2020entropic}, 
\begin{align}
    \pi(G(x), y) = P_{G(X)}(G(x))P_Y(y)e^{\frac{\phi(G(x))+\psi(y)-\|G(x)-y\|_2^2}{\lambda}}.\label{eq:opt_coup}
\end{align}

Even though we are less interested in the computational aspects of optimal transport in this paper, we note that the optimal dual potentials for entropy regularized 2-Wasserstein distance can be shown to satisfy the following equations
\begin{align}
    \phi^*(G(x)) &=-\lambda\ln\EE\l[e^{(\psi^*(Y) - \|Y-G(x)\|_2^2)/\lambda}\r]\label{eq:SK0}\\
    \psi^*(y) &=-\lambda\ln\EE\l[e^{(\phi^*(G(X)) - \|y-G(X)\|_2^2)/\lambda}\r]\label{eq:SK1}
\end{align}
Equations \eqref{eq:SK0},\eqref{eq:SK1} give rise to the celebrated Sinkhorn-Knopp algorithm that allows for fast computation of entropic optimal transport via iterative updates: at iteration $t$ we set,
\begin{align}
    \phi^t(G(x)) &=-\lambda\ln\EE\l[e^{(\psi^{t-1}(Y) - \|Y-G(x)\|_2^2)/\lambda}\r]\label{eq:SK_it}\\
    \psi^t(y) &=-\lambda\ln\EE\l[e^{(\phi^t(G(X)) - \|y-G(X)\|_2^2)/\lambda}\r].\label{eq:SK_it1}
\end{align}

Since Sinkhorn divergence is a linear combination of entropy-regularized Wasserstein distances, it also has a dual form and strong duality holds. The dual form of entropy-regularized 2-Wasserstein distance gives rise to an equivalent formulation of the Sinkhorn divergence as a linear combination of the dual formulations of entropy-regularized 2-Wasserstein distances. Since $W_{2,\lambda}^2(P_{G(X)}, P_{G(X)})$ and $W_{2,\lambda}^2(P_Y,P_Y)$ are symmetric in the dual potentials and concave, the optimal dual potentials will be equal, i.e. $\phi^x(G(x)) = \psi^x(G(x))$ and $\phi^y(y) = \psi^y(y)$ resulting in the following architecture

\begin{center}
	\vspace{-3mm}
	\begin{tikzpicture}[thick,scale=0.8, every node/.style={scale=0.8}]
	\tikzset{box/.style={
			draw, rectangle, rounded corners=0.8pt,
			text width=2.2cm, text height=0.75cm}}
	\tikzset{box_dashed/.style={
			draw, dashed, rectangle, rounded corners=0.8pt,
			text width=2.2cm, text height=0.75cm}}
	\tikzset{node/.style={draw, circle, inner sep=0.5cm}}
	
	\path (0, 0) coordinate [node]
	(x) node {$X\sim P_X$};
	
	\path (2.5, 0) coordinate [box] (G) node {\begin{tabular}{c} Generator\\ $G:\cX\to\cY$ \end{tabular}};
	
	\path (5, -1.2) coordinate [node] (GX) node {$G({X})$};
	\path (5, +1.2) coordinate [node] (Y) node {${Y\sim P_Y}$};
	
	\path (7.5, -0.6) coordinate [box]
	(D_G) node {\begin{tabular}{c} Discriminator\\ $\phi:G(\cX)\to\RR$ \end{tabular}};
	\path (7.5, -1.8) coordinate [box]
	(Dx_G) node {\begin{tabular}{c} Discriminator\\ $\phi^x:G(\cX)\to\RR$ \end{tabular}};
	\path (7.5, +0.6) coordinate [box]
	(D_r) node {\begin{tabular}{c} Discriminator\\ $\psi:\cY\to\RR$ \end{tabular}};
	\path (7.5, +1.8) coordinate [box_dashed]
	(Dy_r) node {\begin{tabular}{c} Discriminator\\ $\phi^y:\cY\to\RR$ \end{tabular}};
	\path (10, 0) coordinate [box,text width=0.75cm]
	(l) node {\begin{tabular}{c}
	     loss\\
	     $S_{\lambda}$
	\end{tabular}};
	
	\draw [->] (Y) -- (D_r);
	\draw [->] (Y) -- (Dy_r);
	\draw [->] (GX) -- (D_G);
	\draw [->] (GX) -- (Dx_G);
	\draw [->] (x) -- (G);
	\draw [->] (D_r) -- (l);
	\draw [->] (D_G) -- (l);
	\draw [->,dashed] (Dy_r) edge[out=0] (l);
	\draw [->] (Dx_G) edge[out=0, in=225] (l);
	\draw [->] (G) edge[out=0] (GX);
	\end{tikzpicture}
\end{center}
The Discriminator $\phi^y$ is dashed since it depends only on the distribution $P_Y$ and does not influence the generator. We refer the reader to \citep{feydy2019interpolating} for details.

\section{Population Solution For the Linear/Gaussian Setting}\label{sec:gauss}
	
	In this section, we aim to compare the optimal solution for GANs when we use the different measures introduced in the previous section for quantifying the dissimilarity between the generated and target probability distributions. For this purpose, we focus on the benchmark setting considered in \citep{feizi2017understanding}, where the generator is linear and the target distribution is Gaussian. In this case, we can rewrite the general formulation  of \eqref{eq:GAN} as:
	\begin{align}
	\min_{G\in \mathbb R^{d\times r}}W_{2}^2\l(\PP_{GX},\PP_Y\r), \label{e:lqgw2gan}
	\end{align}
	where the latent random variable $X\in \mathbb R^r$ follows the standard Gaussian distribution $\mathcal N(0, I_r)$, the underlying distribution of data $Y\in \mathbb R^d$ is $\mathcal N(0, K_Y)$, and the optimization is over all matrices $G\in \mathbb R^{d\times r}$ with $d\geq r$ so that the generated  distribution is $\PP_{GX}$. The population solution to the above W2GAN problem has been characterized in \citep{feizi2017understanding} as the $r$-PCA solution of $Y$, i.e.  the covariance matrix $K_{G^*X}$ for $P_{G^*X}$, where $G^*$ denotes the minimizer of \eqref{e:lqgw2gan}, is a  rank-$r$ matrix whose top $r$ eigenvalues and eigenvectors are the same as those of $K_Y$. 
	
We next show that adding entropic regularization to the W2GAN objective changes this solution to a soft-thresholded $r$-PCA solution of $Y$ as shown by the following theorem.
	
	\begin{theorem}\label{thm:PCA}
		Let $Y\sim \cN(0,\bK_Y)$ and $X\sim \cN(0,\bI_r)$ where $r\leq d$. The population solution $P_{G^*X}$ to the entropic W2GAN problem
			\begin{align}
	\min_{G\in \mathbb R^{d\times r}}W_{2,\lambda}^2\l(\PP_{GX},\PP_Y\r), \label{e:lqgew2gan}
	\end{align}
		is given by a soft-thresholded $r$-PCA solution of $Y$, i.e.,  the covariance matrix $K_{G^*X}$ for $P_{G^*X}$, where $G^*$ denotes now the minimizer of \eqref{e:lqgew2gan}, is a  rank-$r$ matrix whose   top $r$ eigenvectors are the same as those of $\bK_Y$ and the top $r$ eigenvalues are $$\sigma_i^2 = (\lambda_i(\bK_Y) - \lambda/2)_+\quad \textrm{for}\quad i\in[1: r],$$ where $(x)_+:=\max\{x,0\}$ and $\{\lambda_i(\bK_Y)\}_{i=1}^r$ are the top $r$ eigenvalues of $\bK_Y.$  
	\end{theorem}
Note that the population solution for the	entropic W2GAN is not the same as that for the unregularized W2GAN, which is not surprising as they optimize two different objective functions. Nevertheless, Theorem~\ref{thm:PCA} reveals that in the linear/Gaussian case, there is a natural relationship between the two solutions as the former turns out to be a soft-thresholded version of the latter. Note that the soft-thresholding promotes sparsity in the eigenvalues of the covariance matrix of the generated distribution since if many of the eigenvalues of $\bK_Y$ are below the threshold $\lambda/2$ the rank of $K_{G^*X}$ can be significantly smaller than $\bK_Y$. 

We note that soft thresholding of singular values arises as the optimal solution to a different problem that has been studied in the context of low rank matrix completion. Consider the problem:
	\begin{align*}
	\min_{\bK_{Z}\in \mathbb R^{d\times d}} \|\bK_{Z} - \bK_Y \|_F^2 + \lambda \|\bK_{Z}\|_{*} ,
	\end{align*}
where $\|\cdot\|_*$ is the nuclear norm, i.e. the sum of all singular values of a matrix, which can be regarded as a relaxation of a low rank constraint. In \citep[Theorem 2.1]{cai2010singular}, the solution of this problem is shown to be the soft thresholded PCA solution for $r=d,$ i.e. $\bK_{Z}$ and $\bK_Y$ share the same eigenvectors with corresponding eigenvalues thresholded as in Theorem~\ref{thm:PCA}.  From this perspective, entropic regularization in the linear/Gaussian case can be viewed as constraining the nuclear norm of the output covariance matrix, promoting a low rank approximation of the generated distribution. 
	
We next investigate the population solution for the Sinkhorn W2GAN and show that, while it is not the case in general, when restricted to the linear/Gaussian benchmark, surprisingly Sinkhorn W2GAN does recover the regular PCA solution as shown in the following theorem. We remark that this is not a simple consequence of the property 	$S_{\lambda}(\PP_{Y},\PP_Y)=0$ for any $\PP_Y$ of the Sinkhorn divergence, as in the current setting the Sinkhorn divergence between the optimal generated and target distributions is non-zero.  However, it does suggest that the Sinkhorn divergence can lead to solutions closer to the target distribution, while also possessing other favorable qualities like unbiasedness and  sample complexity, as we investigate in the following section.
	\begin{theorem}\label{col:PCA} 
		Let $Y\sim \cN(0,\bK_Y)$ and $X\sim \cN(0,\bI_r)$ where $r\leq d$. The population solution $P_{G^*X}$ to the Sinkhorn W2GAN problem given by
		\begin{align*}
	\min_{G\in \mathbb R^{d\times r}}S_{\lambda}\l(\PP_{GX},\PP_Y\r), 
	\end{align*}
		is given by the $r$-PCA solution of $Y$.
	\end{theorem}
	\subsection{Proofs of Theorems \ref{thm:PCA} and \ref{col:PCA}}
	\begin{proof}[Proof of Theorem \ref{thm:PCA}]  
		Let $Z=GX$, where $G\in \mathbb R^{d\times r}$. Since $X\sim \mathcal N (0, I_r)$,  $P_Z$ is a $d$-dimensional Gaussian distribution whose covariance matrix $K_Z$ has rank less than or equal to $r$. For any such $P_Z$, denote by $\mathcal S_Z$ the $r$-dimensional subspace that contains the support of $Z$. For any $Y\in \mathbb R^d$, let $Y_{\mathcal S_Z}$ and $Y_{\mathcal S_Z^\perp}$ be respectively the projections of $Y$ onto  $\mathcal S_Z$ and its orthogonal complement $\mathcal S_Z^\perp$ so that $Y=Y_{\mathcal S_Z}+Y_{\mathcal S_Z^\perp}$. Note that for a fixed $G$, $Y_{\mathcal S_Z}$ and $Y_{\mathcal S_Z^\perp}$ can be computed given $Y$.  The entropy regularized 2-Wasserstein distance is then
		\begin{align}
		W_{2,\lambda}^2(P_Y,P_Z) &=  \min_{\pi\in\Pi(P_Y,P_{Z})}\mathbb{E}_\pi\!\!\l[\|Z-Y\|^2\r]+\lambda I_{\!\pi}\!(Z;Y) \nonumber\\
		&=\mkern-10mu\min_{\pi\in\Pi(P_Y,P_{Z})}\mkern-14mu\mathbb{E}_\pi\!\bigl[\|(Z-Y_{\mathcal S_Z})-Y_{\mathcal S_Z^\perp}\|^2\bigr]+\lambda I_{\!\pi}\!(Z;Y)\nonumber\\
		&=\mkern-8mu\min_{\pi\in\Pi(P_Y,P_{Z})}\mkern-14mu\mathbb{E}_\pi\!\!\l[\|Z-Y_{\mathcal S_Z}\|^2\r]+\mathbb{E}\bigl[Y_{\mathcal S_Z^\perp}\|^2\!\bigr]+\lambda I_{\!\pi}\!(Z;Y) \label{E:preserve}\\
		&=\mkern-10mu\min_{\pi\in\Pi(P_Y,P_{Z})}\mkern-14mu\mathbb{E}_\pi\!\!\l[\|Z-Y_{\mathcal S_Z}\|^2\r]+\mathbb{E}\bigl[Y_{\mathcal S_Z^\perp}\|^2\!\bigr]+\lambda I_{\!\pi}\!(Z;Y_{\mathcal S_Z}\!).\!\label{E:followedbymatrix} 
		\end{align}
		The last equality above holds because $I_\pi(Z;Y )=I_\pi(Z;Y_{\mathcal S_Z}, Y_{\mathcal S_Z^\perp} )\geq I_\pi(Z;Y_{\mathcal S_Z} )$ and moreover, for any coupling $\pi$, one can construct $\pi'$ such that $\pi'(Z,Y_{\mathcal S_Z}, Y_{\mathcal S_Z^\perp})=\pi(Z,Y_{\mathcal S_Z})\pi(Y_{\mathcal S_Z^\perp}| Y_{\mathcal S_Z})$, i.e. $Z-Y_{\mathcal S_Z}-Y_{\mathcal S_Z^\perp}$ forms a Markov chain. Note that $\pi'$ preserves the values of the first two terms in \eqref{E:preserve} while $I_{\pi'}(Z;Y_{\mathcal S_Z}, Y_{\mathcal S_Z^\perp} )= I_\pi(Z;Y_{\mathcal S_Z} )$.
		
		Consider the optimization problem in the entropic W2GAN, i.e. $\min_{P_Z \in \mathcal N_{d,r}}W_{2,\lambda}^2(P_Y,P_Z)$, where the optimization is over the set $\mathcal N_{d,r}$ of all $d$-dimensional Gaussian distributions with rank not exceeding $r$. In light of \eqref{E:followedbymatrix}, the above is
		\begin{align}
		\!\!\!\min_{\substack{\cS\in\mathbb{S}_d:\text{dim}(\mathcal S)\leq r\\ 
				P_Z \in \mathcal N_{d,r}:\, Z\in \mathcal S\\
				\pi\in\Pi(P_Y,P_{Z})}}
		& \mkern-13mu\mathbb{E}_\pi[\|Z-Y_{\mathcal S }\|^2] +\mathbb{E}[\|Y_{\mathcal S ^\perp}\|^2]+\lambda I_\pi(Z;Y_{\mathcal S }), \label{e:3terms}
		\end{align}
		where $\mathbb{S}_d$ is the set of all subspaces of $\RR^d.$ To solve \eqref{e:3terms} we first fix $\cS.$  If columns of $U\in \mathbb R^{d\times r}$ form an orthonormal basis of $\cS$, i.e. $\cS=\Ima U$ and $U^TU=I_r$, we replace $Z$ and $Y_{\mathcal S}$ in \eqref{e:3terms} by $U^TZ$ and $U^TY$ respectively. To find optimal $\pi,P_Z$ for $\cS$ we then solve
		\begin{align}
		\mkern-7mu\min_{\substack{P_Z \in \mathcal N_{d,r}:Z\in \text{Im}(U)\\
				\pi\in \Pi(P_Z,P_Y)}}\mathbb{E}_\pi&[\|U^TZ-U^TY\|^2] +\lambda I_\pi(U^TZ;U^TY)-\mathbb{E}[\|U^TY\|^2]+\mathbb{E}[\|Y\|^2] \label{e:withinbrk}
		\end{align}
		Let $\bar Z=U^TZ$ and $\bar Y=U^TY,$ and let $\mathcal{N}_{r,r}$ be the set of all $r$-dimensional Gaussian distributions. Then solving Problem~\eqref{e:withinbrk} is equivalent to solving
		\begin{align}
		\min_{P_{\bar{Z}}\in\mathcal{N}_{r,r}}\min_{\pi\in\Pi(P_{\bar Z},P_{\bar Y})}\mathbb{E}_\pi [\|\bar{Z}-\bar{Y}\|^2] + \lambda I_\pi(\bar{Z};\bar{Y})  \label{e:toberatedistortion}
		\end{align}
		
		We will proceed by first lower bounding \eqref{e:toberatedistortion} and then providing the coupling and $P_{\bar Z}$ that achieve the lower bound.

		WLOG, we can assume $\bar{Y}$ has diagonal covariance matrix $\bK_{\bar Y}=\diag(\Lambda_1, \ldots, \Lambda_r)$, where the diagonal elements are in decreasing order. Since $\bar{Y}$ is also Gaussian this implies that its components are independent. This in turn implies that
		$I_\pi(\bar{Z};\bar{Y})\geq \sum_{i=1}^rI_\pi(\bar{Z}_i;\bar{Y}_i),$ and hence \eqref{e:toberatedistortion} can be lower bounded by
		\begin{align}
		&\min_{P_{\bar{Z}}\in\mathcal{N}_{r,r}}\min_{\pi\in\Pi(P_{\bar Z},P_{\bar Y})}\mathbb{E}_\pi [\|\bar{Z}-\bar{Y}\|^2] + \lambda I_\pi(\bar{Z};\bar{Y})\nonumber\\
		&\qquad\geq \min_{P_{\bar{Z}}\in\mathcal{N}_{r,r}}\min_{\pi\in\Pi(P_{\bar Z},P_{\bar Y})}\sum_{i=1}^r
		\mathbb{E}_{\pi_i} [(\bar{Z}_i-\bar{Y}_i)^2] + \lambda I_{\pi_i}(\bar{Z}_i;\bar{Y}_i) \label{e:firstsplit1dim}
		\end{align}
		
		Note that for fixed Gaussian $P_{\bar Z},P_{\bar Y}$ and cross-covariance matrix $K_{\bar Z\bar Y}$ the first term in \eqref{e:toberatedistortion} is fixed and the mutual information term is minimized when $\pi$ is jointly Gaussian, i.e. $\pi\in\cN(P_{\bar Z}, P_{\bar Y}),$ where $\cN(\mu,\nu)$ denotes a set of jointly Gaussian distributions with marginals $\mu, \nu.$ Therefore the minimization in \eqref{e:firstsplit1dim} can be restricted to $\pi\in\cN(P_{\bar Z}, P_{\bar Y}).$ If we let $D_i = \EE(\bar Y_i - \bar Z_i)^2,$ this in turn yields
		\begin{align}
		    I_{\pi_i}(\bar{Z}_i;\bar{Y}_i)& = h(\bar{Y}_i) - h(\bar{Y}_i|\bar{Z}_i)\nonumber\\
		    &=h(\bar{Y}_i) - h(\bar{Y}_i-\bar{Z}_i|\bar{Z}_i)\nonumber\\
		    &\geq h(\bar{Y}_i) - h(\bar{Y}_i-\bar{Z}_i)\label{e:conditioningreducesentropy}\\
		    & = (1/2)\ln(\Lambda_i/D_i),\nonumber
		\end{align}
		where \eqref{e:conditioningreducesentropy} follows from the fact that conditioning reduces entropy. 
		Since mutual information is non-negative,
		we can tighten the bound to $I_{\pi_i}(\bar{Z}_i;\bar{Y}_i)\geq \max\{0, (1/2)\ln(\Lambda_i/D_i\}.$
		
		Thus, continuing the lower bound from \eqref{e:firstsplit1dim} we get
		\begin{align}
		&\min_{P_{\bar{Z}}\in\mathcal{N}_{r,r}}\min_{\pi\in\Pi(P_{\bar Z},P_{\bar Y})}\mathbb{E}_\pi [\|\bar{Z}-\bar{Y}\|^2] + \lambda I_\pi(\bar{Z};\bar{Y})\nonumber\\
		&\qquad \geq \sum_{i=1}^r\min_{D_i\geq0}D_i + (\lambda/2)\max\{\ln(\Lambda_i/D_i),0\}\nonumber\\
		&\qquad = \sum_{i=1}^r \min_{0\leq D_i\leq\Lambda_i}D_i + (\lambda/2)\ln(\Lambda_i/D_i)\label{e:commonsensesimplification}\\
		&\qquad = \sum_{i=1}^r \min_{0\leq D_i\leq\Lambda_i}g(D_i) + (\lambda/2)\ln(\Lambda_i),\nonumber
		\end{align}
		where \eqref{e:commonsensesimplification} follows from the fact that increasing $D_i$ beyond $\Lambda_i$ leaves the second summand the same, while increases the first one, so the minimum is attained at $D_i\leq \Lambda_i$ and
        $g(x) = x - \lambda/2\ln x.$ As $g'(x)= 1 - \lambda /(2x)<0$ for $x<\lambda/2,$  the optimal value is attained at $D_i = D_i^* = \min\{\lambda/2, \Lambda_i\}$ and plugging it into the bound we get 
		\begin{align}
		&\min_{P_{\bar{Z}}\in\mathcal{N}_{r,r}}\min_{\pi\in\Pi(P_{\bar Z},P_{\bar Y})}\mathbb{E}_\pi [\|\bar{Z}-\bar{Y}\|^2] + \lambda I_\pi(\bar{Z};\bar{Y})\label{e:tightratedistortionbound}\\
		&\qquad \geq\sum_{i=1}^r\min\{\lambda/2, \Lambda_i\} + (\lambda/2)\ln(\Lambda_i/\min\{\lambda/2, \Lambda_i\})\label{e:ratedistortionsolution}
		\end{align}
		
		To prove that the lower bound holds with equality we need \eqref{e:conditioningreducesentropy} to hold with equality, i.e. $\bar Y_i-\bar Z_i$ be independent of $\bar Z_i,$ so we can choose $\bar Y = \bar Z + N,$ where $N\sim \cN(0, \diag(\{\min\{\lambda/2, \Lambda_i\}\}_{i=1}^r)$ independent of $\bar Z.$ Note that since the distribution of $\bar Y$ is fixed, this in turn fixes the distribution of $\bar Z$. With this choice of a coupling 
		\begin{align*}
		    \EE_{\pi}[\|\bar Z - \bar Y\|^2] &= \EE_{\pi}[\|N\|^2] = \sum_{i=1}^r D_i^*\\
		    I(\bar Z; \bar Y) &=  h(\bar Y)  - h(\bar Y - \bar{Z}\mid \bar{Z})\\
		    &=h(\bar Y) - h(N)\\
		    &= (1/2)\sum_{i=1}^r\ln(\Lambda_i)-\ln(D^*_i)
		\end{align*}
		Combining the above we get 
		\begin{align*}
		    \mathbb{E}_\pi [\|\bar{Z}-\bar{Y}\|^2] + \lambda I_\pi(\bar{Z};\bar{Y}) =\sum_{i=1}^r \min\{\lambda/2, \Lambda_i\} +  (\lambda/2)\sum_{i=1}^r\ln(\Lambda_i/\min\{\lambda/2, \Lambda_i\}),
		\end{align*}
		which matches the lower bound in \eqref{e:tightratedistortionbound}, so $\bar Y = \bar Z + N$ with $N\sim\cN(0,\diag(\{D_i\}_{i=1}^r))$ independent of $Z$ is the optimal coupling.\\
		
		Before we continue with the proof we make the following remark. Note that when the condition $P_{\bar{Z}}\in\mathcal{N}_{r,r}$ is dropped we get a lower bound on \eqref{e:toberatedistortion}:
		\begin{align}
		W_{lower} = \min_{\pi_{\bar{Z}|\bar{Y}}}\mathbb{E}_\pi [\|\bar{Z}-\bar{Y}\|^2] + \lambda I_\pi(\bar{Z};\bar{Y})  \label{e:wassersteinlowerbound}
		\end{align}
	Introduction of a new variable $D = \EE \|\bar{Z}-\bar{Y}\|^2$ leads to the following  optimization problem \label{e:ratedistortion}:
		\begin{equation*}
            \begin{array}{ll}
            \underset{\pi_{\bar{Z}|\bar{Y}},D}{\min}   & D + \lambda I_\pi(\bar{Z};\bar{Y})\\
            \mbox{subject to} & D = \EE[\|\bar{Z}-\bar{Y}\|^2]\\
            & D\geq 0
            \end{array}
            \end{equation*}
        Note that the equality in the constraints can be relaxed to an inequality since the objective is linear in $D,$ which leads to 
		\begin{equation}
            \begin{array}{ll}
            \underset{\pi_{\bar{Z}|\bar{Y}},D}{\min}   & D + \lambda I_\pi(\bar{Z};\bar{Y})\\
            \mbox{subject to} &  D \geq \EE[\|\bar{Z}-\bar{Y}\|^2]\\
            & D\geq 0
            \end{array}
            \label{e-opt-prob}
            \end{equation}
        Finally we can rewrite the full minimization over $\pi_{\bar{Z}|\bar{Y}},D$ as a consecutive one, which leads to
		\begin{align}
		    W_{lower} =\min_{D\geq 0} \l\{D + \lambda \min_{\pi_{\bar{Z}|\bar{Y}}:  \EE[\|\bar{Z}-\bar{Y}\|^2]\leq D}I_\pi(\bar{Z};\bar{Y})\r\}
		    \label{e:sequantialoptlowerbound}
		\end{align}
        The inner minimization problem is exactly the Gaussian rate distortion problem \citep[Eq. (10.38)]{coverbook} that can be solved by noting that the mutual information term is minimized for a Gaussian distribution, plugging the value of the mutual information in and writing down the Karush-Kuhn-Tucker optimality conditions.
		The solution for this problem is given by reverse waterfilling \citep[Theorem 10.3.3]{coverbook}, under which the optimal $P_{\bar Z}$ is an $r$-dimensional Gaussian which matches the solution we have obtained in \eqref{e:ratedistortionsolution}.

		The entropic W2GAN optimization problem \eqref{e:3terms} is then equivalent to minimizing \eqref{e:ratedistortionsolution} over the set of all $r-$dimensional subspaces of $\RR^d:$
		\begin{align}
		\min_{U\in\mathbb{R}^{d\times r}}&\sum_{i=1}^r \min\{\Lambda_i, \lambda/2\} + \frac{\lambda}{2}\ln\frac{\Lambda_i}{\min\{\Lambda_i, \lambda/2\}}+ \EE[\|Y_{(\text{Im}\,U)^{\perp}}\|^2]\nonumber\\
		&=\min_{U\in\mathbb{R}^{d\times r}}\sum_{i=1}^r \biggl(\Lambda_i+ \lambda/2-\max\{\Lambda_i, \lambda/2\} + \frac{\lambda}{2}\ln\frac{\max\{\Lambda_i,\lambda/2\}}{\lambda/2}\biggr)+ \EE[\|Y_{(\text{Im}\,U)^{\perp}}\|^2]\nonumber\\
		&= \min_{U\in\mathbb{R}^{d\times r} }\sum_{i=1}^r\left(\frac{\lambda}{2}\ln\frac{\max\{\Lambda_i,\lambda/2\}}{\lambda/2}-\max\{\Lambda_i,\lambda/2\}\right)+ \frac{r\lambda}{2} + \EE[\|Y\|^2]\nonumber
		\end{align}
		where the optimization is over all $U\in\mathbb{R}^{d\times r}$ such that $U^TU=I_r$ and $U^T\bK_Y U=\text{diag}(\Lambda_1,\ldots, \Lambda_r)$. We now let
		\begin{align*}
		f(\Lambda_i)=(\lambda/2)\ln\left(\max\{\Lambda_i,\lambda/2\}/\l(\lambda/2\r)\right)-\max\{\Lambda_i,\lambda/2\},
		\end{align*} 
		and complete the proof by showing 
		\begin{align}
		\min_{U\in\mathbb{R}^{d\times r}:U^T\bK_Y U=\text{diag}(\Lambda_i,\ldots, \Lambda_r)}\sum_{i=1}^rf(\Lambda_i)=\sum_{i=1}^rf(\lambda_i(\bK_Y)).\label{eq:majorizing},
		\end{align} 
		where $[\lambda_1(\bK_Y),\ldots, \lambda_r(\bK_Y)]$ are the largest $r$ eigenvalues of $\bK_Y$.
		
		Indeed, for each $U$ we can construct an orthogonal matrix $U' = [U\; U_{\perp}]\in\RR^{d\times d}$ with the first $r$ columns equal to $U.$ Then the first $r$ diagonal elements of $U'^TK_YU'$ are $\Lambda_1\ldots\Lambda_r,$ let the rest be $\Lambda_{r+1}\ldots\Lambda_d.$ The eigenvalues of $U'^TK_YU'$ are $\lambda_1(K_Y)\ldots \lambda_d(K_Y).$ By the fact that the diagonal entries of a symmetric matrix are majorized by its diagonal entries~\citep[Theorem 9.B.1]{marshall1979inequalities}, we have 
		$$\{\Lambda_i\}_{i=1}^d\prec \{\lambda_i(K_Y)\}_{i=1}^d,$$ where $\prec$ denotes majorization, i.e.  for $x, y\in \RR^d:$ $$x\prec y \iff \forall r\leq d:\;\sum_{i=1}^r x_i\leq \sum_{i=1}^r y_i\textrm{ and } \sum_{i=1}^d x_i= \sum_{i=1}^d y_i.$$

		Therefore, $$\sum_{i=1}^{r'}\Lambda_i\leq \sum_{i=1}^{r'}\lambda_i(K_Y),\quad\forall r'\leq r$$  and $$\{\Lambda_i\}_{i=1}^r\prec_w \{\lambda_i(K_Y)\}_{i=1}^r,$$ where $\prec_w$ denotes weak majorization, i.e. for $x, y\in\RR^r:$ $$x\prec_w y \iff \forall r'\leq r:\;\sum_{i=1}^{r'} x_i\leq \sum_{i=1}^{r'} y_i.$$ 
		
		We can now use the majorizing inequality from \citep[Proposition 4.B.2]{marshall1979inequalities} to complete the proof. 
		\begin{proposition}[Majorizing inequality]
		
		     The inequality
		     \begin{align}
		         \sum g(x_i) \leq \sum g(y_i)
		         \label{e:majorizinglemma}
		     \end{align}
            holds for all continuous non-decreasing convex functions g if and only if $x\prec_w y$.
		\end{proposition}
		
		Note that $-f$ is a continuous non-decreasing convex function and $$\{\Lambda_i\}_{i=1}^r\prec_w \{\lambda_i(K_Y)\}_{i=1}^r,$$ thus by the proposition 
		$$\sum_{i=1}^rf(\lambda_i(\bK_Y))\leq \sum_{i=1}^rf(\Lambda_i).$$ 
		Therefore, columns of the optimal $U$ are the top $r$ eigenvectors of $\bK_Y$, and the optimal $P_Z$ has covariance matrix given by
		$$K_Z = U [ 
		\text{diag}(\sigma_1^2,\ldots, \sigma_r^2) | 0_{r\times (d-r)}
		]U^T $$
		where $\sigma_i^2=(  \lambda_i(\bK_Y)   -\lambda/2)_+$.
	\end{proof}

	\smallbreak
	\begin{proof}[Proof of Theorem \ref{col:PCA}]
		From \eqref{E:followedbymatrix} in the proof of Theorem~\ref{thm:PCA}, we have for given $G$ and $\mathcal S = \Ima G$, 
		\begin{align*}
		&W_{2,\lambda}^2(\PP_{Z},\PP_Y)- \EE[\|Y_{\cS^\perp}\|^2]=\min_{\pi\in \Pi(\PP_{Z},\PP_{Y_\cS})}\EE[\|Z-Y_\cS\|^2] +\lambda I(Z;Y_\cS)=W_{2,\lambda}^2(\PP_{Z},\PP_{Y_\cS}),
		\end{align*}
		and therefore for the Sinkhorn divergence,
		\begin{align}
		S_{\lambda}(\PP_{Z},\PP_Y)-\EE\|Y_{\cS^\perp}\|_2^2 &=W_{2,\lambda}^2(\PP_{Z},\PP_{Y_\cS}) - \bigl(W_{2,\lambda}^2(\PP_{Z},\PP_{Z})+W_{2,\lambda}^2(\PP_{Y},\PP_{Y})\bigr)/2\nonumber \\
		&=S_{\lambda}(\PP_{Z},\PP_{Y_\cS}) +
		\bigl(W_{2,\lambda}^2(\PP_{Y_\cS},\PP_{Y_\cS})-W_{2,\lambda}^2(\PP_{Y},\PP_{Y})\bigr)/2\label{e:sinkhorngan},
		\end{align}
		which follows from the definition of Sinkhorn divergence.
		
		Consider the optimization problem in the Sinkhorn divergence GAN, i.e. $\min_{P_Z}S_{\lambda}(\PP_{Z},\PP_Y)$. In light of \eqref{e:sinkhorngan}, given $Z\in \mathcal S$ the optimal $P_Z$ should be $P_{Y_\cS}$,  which makes the first term in \eqref{e:sinkhorngan} zero while the remaining terms do not depend on $P_Z$. Therefore, it only remains to optimize over $\mathcal S$, and in particular, the problem reduces to 
		\begin{align}
		\min_{\mathcal S \in \mathbb S_d: \text{dim}(\mathcal S)\leq r}   W&_{2,\lambda}^2(\PP_{Y_\cS},\PP_{Y_\cS})/2 + \EE\|Y_{\cS^\perp}\|_2^2. \label{e:sinkhornganminimizeS}
		\end{align}
		To calculate $W_{2,\lambda}^2(\PP_{Y_\cS},\PP_{Y_\cS})$ we use~\citep[Theorem 1]{janati2020entropic} stated below.
		\begin{proposition}[Entropy-regularized Wasserstein distance for Gaussian measures]
		
		Let $K_X, K_Y\in \RR^{d\times d}$ be positive definite and $X \sim \cN(\mu_X, K_X)$ and $Y \sim \cN(\mu_Y, K_Y)$. Define $\bD_\lambda = (4A^{\frac{1}{2}} K_Y A^{\frac{1}{2}} +\lambda^2 I/4)^{\frac{1}{2}}$. Then,
        \begin{align*}
            W_{2,\lambda}^2(\alpha, \beta)=& \|\mu_X - \mu_Y\|^2 + \Tr(K_X) + \Tr(K_Y) - \Tr(\bD_\lambda)\\
        & + \frac{\lambda}{2}\l(d(1 - \log\lambda)  + \log\det\left(\bD_\lambda + \frac{\lambda}{2}\right)\r)
        \end{align*}
		\end{proposition}
		
		With this proposition the objective function in \eqref{e:sinkhornganminimizeS} becomes
		\begin{align*}
		W_{2,\lambda}^2(\PP_{Y_\cS},\PP_{Y_\cS})/2 + \EE\|Y_{\cS^\perp}\|_2^2
		&= \Tr \bK_{Y_{\cS^\perp}}+\Tr \bK_{Y_\cS}- \Tr\bigl((4\bK_{Y_\cS}^2+\lambda^2 I/4)^{1/2}\bigr)/2\\
		&\qquad +\lambda\ln\det\bigl(\l(4\bK_{Y_\cS}^2+\lambda^2 I/4\r)^{1/2}+\lambda I/2\bigr)/4 + C\\
		&=\sum_{i=1}^{r}\!\biggl(\frac{\lambda}{4}\ln\!\l(\!\!\sqrt{4\Lambda_i^2+\frac{\lambda^2}{4}}+\frac{\lambda}{2}\r) -\frac{1}{2}\sqrt{4\Lambda_i^2+\frac{\lambda^2}{4}}\biggr)\!+\!
		C'
		\end{align*}
		where $\Lambda_i$ is the $i$th eigenvalue of $U^T\bK_{Y_\cS}U$ for some $U \in \mathbb R^{d\times r}$ such that $\Ima U =\mathcal S$ and $U^T U=I_r$, $C$ is a constant depending only on $\lambda$ and $d$ and $C'=\Tr\bK+C.$  
		The above is minimized when $\Lambda_{i} = \lambda_i(\bK_Y)$ using a similar argument as the one for showing \eqref{eq:majorizing}, i.e., by using the majorizing inequality and noting that for the function
		$$f(x) =\sqrt{4x^2+\lambda^2/4}/2- \lambda\ln(\sqrt{4x^2+\lambda^2/4}+\lambda/2)/4,$$
		$-f(x)$ is convex and non-decreasing for $\lambda>0$ and $x\geq0.$
	\end{proof}

\section{Generalization Error of Empirical Solution}\label{sec:generalization}

In this section we discuss the generalization capability of the empirical solutions for W2GAN, entropic W2GAN  and Sinkhorn W2GAN, respectively. Note that in the population case, the underlying distribution of data $P_Y$ was known in the GAN formulations \eqref{eq:GAN}, \eqref{eq:EGAN} and \eqref{eq:OTsinkgan}. In contrast, here we consider the finite sample case, where empirical distribution $Q_Y^n$ extracted from sample $\hat\cY=\{y_i\}_{i=1}^n$ is used in the GAN objective \eqref{eq:GAN}, \eqref{eq:EGAN} and \eqref{eq:OTsinkgan} to approximate $\PP_Y.$ We are interested in how fast the empirical solution $P_{G_n(X)}$ converges to the population solution $P_{G^*(X)}$. 

It was shown  in \citep{feizi2017understanding}  that even in our simple benchmark when generators are linear and data distribution is high-dimensional Gaussian, the convergence for W2GAN is slow in the sense that the generalization error 
$$\EE \bigl[W_{2}^2(\PP_{G_n(X)},\PP_Y)- W_{2}^2(\PP_{G^*(X)},\PP_Y)\bigr] = \Omega(n^{-2/d}).$$

That is to decrease the generalization error by a constant factor the number of samples has to be increased by a factor of $e^{\Omega(d)}$, and hence the generalization capability of W2GAN suffers from the curse of dimensionality. To overcome this, \citep{feizi2017understanding} proposed to constrain the set of discriminators for W2GAN to quadratic. This was motivated by the observation that constraining the discriminator to be quadratic will not affect the population solution because the optimal discriminator for W2GAN is indeed quadratic in the Gaussian setting. On the other hand, it was shown that this constraint will lead to fast convergence to the optimal solution of \eqref{eq:OT} when generators are linear and data distribution is high-dimensional Gaussian. The convergence is of order $O_d(n^{-1/2})$ and  hence the issue of curse of dimensionality in this case is resolved .   

While constraining the discriminator to be quadratic as done in \citep{feizi2017understanding} is conceptually appealing and works for the setup of linear generators and Gaussian data, this insight does not generalize to other distributions, i.e. for non-Gaussian data the generator obtained under a quadratic discriminator is not necessarily the one minimizing the 2-Wasserstein distance between the generated and the target distributions and in general can be far from optimality.
Theorems \ref{thm:generalization} and \ref{thm:generalization_lip} below show that under mild conditions on the underlying distribution of data, the latent random variable and the set of generators, similar convergence can be achieved for entropic W2GAN and Sinkhorn W2GAN without the need to constrain the discriminator.  

To formally state the results, let us first recall some definitions. A distribution $\PP_X$ is $\sigma^2$  sub-gaussian for  $\sigma\geq0$ if $$\EE\exp\l(\|X\|^2/(2r\sigma^2)\r)\leq 2.$$ 
Let $$\sigma^2(X) \!\!=\! \min\{\sigma\!\!\geq0\!\l\vert \EE\exp(\|X\|^2/(2r\sigma^2))\leq 2\r.\},$$ 
denote the sub-gaussian parameter of the distribution of $X.$ A set of generators $\cG$ is said to be star-shaped with a center at $0$ if a line segment between $0$ and $G\in\cG$ also lies in $\cG,$ i.e. 
\begin{align}
G\in\cG\; \Rightarrow \alpha G\in\cG, \forall\alpha\in[0,1]. \label{eq:assump_lin}
\end{align}
Note that this includes the set of all linear generators considered in the last section as a trivial case, as well as the set of linear functions with a bounded norm or a fixed dimension. This also includes the set of all L-Lipschitz functions as another example.
\begin{theorem}\label{thm:generalization}
	Let $\PP_{K_X^{-1/2}X}$ and $\PP_Y$ be sub-gaussian and the generator set $\cG$ be a set of linear functions  satisfying  condition~\eqref{eq:assump_lin}. Then the generalization error for entropic W2GAN can be bounded by
	\begin{align*}
	\EE \bigl[W_{2,\lambda}^2(\PP_{G_n(X)},\PP_Y)&- W_{2,\lambda}^2(\PP_{G^*(X)},\PP_Y)\bigr]\leq K_d\lambda n^{-1/2}\bigl(1+(2\tau^2/\lambda)^{\lceil 5d/4\rceil+3}\bigr),
	\end{align*}
	where $\tau^2 = \max\{\sigma^2(K_X^{-1/2}X)\sigma^2(Y),\sigma^2(Y)\}$ and $K_d$ is a dimension dependent constant.
\end{theorem}
Theorem \ref{thm:generalization} essentially says that under certain mild conditions, the generalization error for entropic W2GAN converges to zero at speed $O_d(1/\sqrt{n})$.  This improved sample complexity suggests that the set of possible discriminators may be implicitly constrained due to the entropic regularization term used in the primal form of entropy regularized 2-Wasserstein distance. Similar results also hold for the set of Lipschitz  functions  and extend to the Sinkhorn W2GAN.
\begin{theorem}\label{thm:generalization_lip}
	Let $\PP_X$ and $\PP_Y$ be sub-Gaussian and the set of generators $\cG$ consist of $L$-Lipschitz  functions, i.e. $\|G(X_1) - G(X_2)\|\leq L\|X_1-X_2\|$ for any $X_1,X_2$ in the support of $P_X$ and let $\cG$ satisfy \eqref{eq:assump_lin}. Then the generalization error for entropic W2GAN
	\begin{align*}
	&\EE \bigl[W_{2,\lambda}^2(\PP_{G_n(X)},\PP_Y)- W_{2,\lambda}^2(\PP_{G^*(X)},\PP_Y)\bigr]		
	\end{align*}
	and that for Sinkhorn W2GAN 
	\begin{align*}
	&\EE \bigl[S_{\lambda}(\PP_{G_n(X)},\PP_Y)- S_{ \lambda}(\PP_{G^*(X)},\PP_Y)\bigr]		\end{align*}
	can be both upper bounded by
	\begin{align}\label{eq:conv}
	\bK_d\lambda n^{-1/2}\bigl(1+(2\tau^2/\lambda)^{\lceil 5d/4 \rceil+3}\bigr)
	\end{align}
	with $\tau^2 = \max\{L^2\sigma^2(X),\sigma^2(Y)\}$.
\end{theorem}
%
It is worth mentioning that a similar result was proved in \citep{luise2020generalization}, however it requires significantly stronger conditions for the set of generator functions $\cG.$ In particular, it does not apply to $\cG$ being the set of all linear functions.
	\subsection{Proofs}
	
	
	
	We first provide the proof of Theorem ~\ref{thm:generalization}. It builds on several results that appear in Theorem 2 of \citep{mena2019statistical} and several lemmas that we summarize below. 
	
	In this section we use $\QQ_Y^n$ to denote the random empirical distribution extracted from a sample $\hat\cY=\{y_i\}_{i=1}^n\sim P_Y^{\otimes n}$ unless stated otherwise. We also use $\EE_{\hat\cY}[\cdot]$ denotes the expectation conditioned on the sample $\hat\cY$ and let  $$\sigma^2_{\hat\cY}(\hY)=\min\{\sigma\geq0\l\vert \EE_{\hat\cY}\exp(\|\hY\|^2/(2r\sigma^2))\leq 2\r.\}$$ be the sub-gaussian parameter of the distribution of $\hY$ conditioned on the sample. 
	
	We note that the definition of the Entropic Wasserstein distance in this paper and \citep{mena2019statistical} differs by a factor of $1/2$ in the first summand, so all the results in \citep{mena2019statistical} are stated  for $(1/2)W_{2, 2\lambda}^2(P_{X}, P_{Y}).$ We state a modification of \citep[Theorem 2]{mena2019statistical} below.
	\begin{proposition}\label{thm:mena_res}
		If $\PP_X$ and $\PP_Y$ are $\sigma^2$ sub-gaussian, then
		\begin{align}
		\EE&\l[\l\vert W_{2,\lambda}^2(\PP_{X},\QQ_Y^n)- W_{2,\lambda}^2(\PP_X,\PP_Y)\r\vert\r]\leq  K_d\lambda n^{-1/2}\bigl(1+(2\sigma^2/\lambda)^{\l\lceil 5d/4\r\rceil+3}\bigr)\label{eq:discrep},
		\end{align}
		where $K_d$ is a constant depending on the dimension.
	\end{proposition}
	\begin{proof}
	    The final step of the proof of \citep[Theorem 2]{mena2019statistical} proves that for $\lambda = 2:$
	    \begin{align}
		(1/2)\EE&\l[\l\vert W_{2,2}^2(\PP_{X},\QQ_{Y}^n)- W_{2,2}^2(\PP_{X},\PP_{Y})\r\vert\r]\leq  K_d n^{-1/2}\bigl(1+(\sigma^2)^{\l\lceil 5d/4\r\rceil+3}\bigr)\label{eq:discrepnolambda},
		\end{align}
		Similar to \citep[Corollary 1]{mena2019statistical} note the connection between $W_{2,\lambda}^2(P_X, P_Y)$ and $W_{2,2}^2:$
		\begin{align}
		W_{2,\lambda}^2(P_X,P_Y) &= \inf_{\pi\in\Pi(P_{X}, P_Y)}\EE[\|X - Y\|^2] + \lambda I_{\pi}(X, Y)\nonumber\\
		& = \frac{\lambda}{2}\hspace{-1pt}\inf_{\pi\in\Pi(P_{X}, P_Y)}\hspace{-1pt}\bigg\{\EE\l[\l\|\frac{X}{\sqrt{\lambda/2}} - \frac{Y}{\sqrt{\lambda/2}}\r\|^2\r]+ 2I_{\pi}\l(\frac{X}{\sqrt{\lambda/2}}, \frac{Y}{\sqrt{\lambda/2}}\r)\bigg\}\nonumber\\
		& = (\lambda/2) W_{2,2}^2(P_{X/\sqrt{\lambda/2}},P_{Y/\sqrt{\lambda/2}})\nonumber
		\end{align}
		From the definition of the sub-gaussian parameter we note that for any $a\in \RR:$ $\sigma^2(aX) = a^2\sigma^2(X),$ thus $\sigma^2(X/\sqrt{\lambda/2}) = 2\sigma^2(X)/\lambda.$ Applying \eqref{eq:discrepnolambda} with sub-gaussian parameter $2\sigma^2/\lambda$ we get the statement of the proposition:
		\begin{align}
		\EE&\l[\l\vert W_{2,\lambda}^2(\PP_{X},\QQ_Y^n)- W_{2,\lambda}^2(\PP_X,\PP_Y)\r\vert\r]\nonumber\\
		&\qquad=\frac{\lambda}{2}\EE\l[\l\vert W_{2,2}^2\l(\PP_{\frac{X}{\sqrt{\lambda/2}}},\QQ_{\frac{Y}{\sqrt{\lambda/2}}}^n\r)- W_{2,2}^2\l(\PP_{\frac{X}{\sqrt{\lambda/2}}},\PP_{\frac{Y}{\sqrt{\lambda/2}}}\r)\r\vert\r]\nonumber\\
		&\qquad\leq  \lambda K_d n^{-1/2}\bigl(1+(2\sigma^2/\lambda)^{\l\lceil 5d/4\r\rceil+3}\bigr).\nonumber
		\end{align}
	\end{proof}
	To prove the theorem we will need several intermediate results from the proof of \citep[Theorem 2]{mena2019statistical} and we  state them below.
    \begin{proposition}[Proposition 2 of \citep{mena2019statistical}] 
        Let $P_X, P_Y$ and $\QQ_Y^n$ all be $\tilde \sigma^2$ sub-gaussian distributions for a possibly random $\tilde \sigma^2\in [0;+\infty).$ Then for a set of functions $F$ denoting 
    $\|\PP_Y - \QQ_Y^n\|_{F} = \sup_{f\in F}|\EE_{Y\sim \PP_Y}[f(Y)] - \EE_{\hY\sim \QQ_Y^n}[f(\hY)]|$ it holds that
    \begin{align}
        &(1/2)\l\vert W_{2,2}^2(\PP_{X},\QQ_{Y}^n)- W_{2,2}^2(\PP_{X},\PP_{Y})\r\vert\leq 2\|\PP_Y - \QQ_Y^n\|_{\cF_{\tilde\sigma}}\nonumber,
    \end{align}
    where $\cF_{\tilde\sigma}$ is a set of functions satisfying for some constants $C_{k, d},$ depending on $k$ and $d$ only and
    for any multi-index $\alpha$ with $|\alpha|=k$
    \begin{align*}
        |D^{\alpha}(f - (1/2)\|\cdot\|^2)(x)|\leq C_{k,d}'
        \begin{cases}
            1+\sigma^4 & \text{if } k = 0\\
            \sigma^k(\sigma + \sigma^2)^k &\text{otherwise}
        \end{cases}
    \end{align*}
    if $\|x\|\leq \sqrt{d}\sigma$ and 
    \begin{align*}
        &|D^{\alpha}(f - (1/2)\|\cdot\|^2)(x)|\\
        &\qquad\leq C_{k,d}'
        \begin{cases}
            1+(1+\sigma^2)\|x\|^2 & \text{if } k = 0\\
            \sigma^k(\sqrt{\sigma\|x\|} + \sigma\|x\|)^k &\text{otherwise}
        \end{cases}
    \end{align*}
    if $\|x\|> \sqrt{d}\sigma.$
    \end{proposition}
    
    The above proposition cannot be used directly for proving the result since the set of functions depends on the random sub-gaussian parameter. To overcome that we state the following proposition that simplifies the previous result.
    \begin{proposition}\label{prop:functioonalbound}
    Let $P_X, P_Y$ and $\QQ_Y^n$ all be $\tilde \sigma^2$ sub-gaussian distributions for a possibly random $\tilde \sigma^2\in [0;+\infty).$ Let for $s\geq 2$ $\cF^s$ be a set of functions satisfying 
    \begin{align*}
        |f(x)|&\leq C_{s, d}(1 + \|x\|^2)\\
        |D^{\alpha}f(x)|&\leq C_{s, d}(1 + \|x\|^s)\,\forall \alpha:\,|\alpha|\leq s
    \end{align*}
    for some constant $C_{s, d}$ that depends only on $s, d.$ Then
    \begin{align}
        &\l\vert W_{2,2}^2(\PP_{X},\QQ_{Y}^n)- W_{2,2}^2(\PP_{X},\PP_{Y})\r\vert\leq 4\|\PP_Y - \QQ_Y^n\|_{\cF^s}(1 + \tilde\sigma^{3s})\nonumber
    \end{align}
    \end{proposition}
    The proof of the proposition follows some of the steps of the proof of \citep[Theorem 2]{mena2019statistical} and is provided here for completeness.
    \begin{proof}
        Note that for large enough constants $C_{s, d}$ ($C_{s, d}\propto d + \max_{k\leq s}2^k C_{k, d}',$ where $C_{k,d}'$ come from the definition of $\cF_{\tilde\sigma^2}$) for any $f\in \cF_{\tilde{\sigma}^2}$ it holds that $\frac{1}{1 + \tilde{\sigma}^{3s}}f\in \cF^s.$ Combining this with Proposition \ref{prop:functioonalbound} we get
        \begin{align}
            &\l\vert W_{2,2}^2(\PP_{X},\QQ_{Y}^n)- W_{2,2}^2(\PP_{X},\PP_{Y})\r\vert\nonumber\\
            &\quad\leq 4\|\PP_Y - \QQ_Y^n\|_{\cF_{\tilde\sigma}}\nonumber\\
            &\quad = (1 + \tilde\sigma^{3s})\hspace{-1.5pt}\sup_{f\in \cF_{\tilde\sigma}}\hspace{-1.5pt}\Big\vert\EE_{Y\sim \PP_Y}\Big[\frac{f(Y)}{1 + \tilde\sigma^{3s}}\Big] - \EE_{\hY\sim \QQ_Y^n}\Big[\frac{f(\hY)}{1 + \tilde\sigma^{3s}}\Big]\Big\vert\nonumber\\
            &\quad\leq (1 + \tilde\sigma^{3s})\sup_{f\in \cF^{s}}|\EE_{Y\sim \PP_Y}[f(Y)] - \EE_{\hY\sim \QQ_Y^n}[f(\hY)]|\nonumber\\
            &\quad =   4\|\PP_Y - \QQ_Y^n\|_{\cF^s}(1 + \tilde\sigma^{3s})\nonumber
        \end{align}
    \end{proof}
    The proof of \citep[Theorem 2]{mena2019statistical} also uses a covering number for $\cF_s$ is used to bound $\EE[\|\PP_Y - \QQ_Y^n\|_{\cF^s}^2].$ Since the result will be used in the proofs of Theorems \ref{thm:generalization} and \ref{thm:generalization_lip}, we will state it here.
    \begin{proposition}\citep[Proof of Theorem 2, page 8]{mena2019statistical}, \citep{gine2021mathematical} For $s = \lceil d/2\rceil + 1,$ for $P_Y$ being $\sigma^2$ sub-gaussian and $\cF_s$ defined in proposition \ref{prop:functioonalbound} it holds that
    \begin{align*}
        \EE[\|\PP_Y - \QQ_Y^n\|^2_{\cF^s}]\leq 
        C_d\frac{1}{n}(1+\sigma^{2d+4})
    \end{align*}
    \label{prop:coveringbound}
    \end{proposition}
    Finally we state here \citep[Lemma 4]{mena2019statistical} that helps bound the even moments of the sub-gaussian parameter $\tilde\sigma^2$ of the (random) empirical distribution $Q_Y^n.$
    \begin{proposition}\citep[Lemma 4]{mena2019statistical}\label{prop:sigmamoments}
    If $Y$ is $\sigma^2$ sub-gaussian then $Q_Y^n$ is $\tilde\sigma^2$ sub-gaussian with
    \begin{align*}
        \EE[\tilde\sigma^{2k}]\leq 2k^k\sigma^{2k}
    \end{align*}
    for any positive integer k,
    \end{proposition}
	To prove the theorem we will also need the following lemmas connected to the properties of $G(X).$
	\begin{lemma}\label{lemma:G_bounded}
        Under the assumption \eqref{eq:assump_lin} the optimal generator 
        \[G^*= \arg\min_{G\in\cG}W_{2,\lambda}^2(\PP_{G(X)}, \PP_Z)\] for $Z\in \RR^d$
        satisfies $\EE\l[\|G^*(X)\|_2^2\r]\leq \Tr\bK_{Z},$ and if $G^*$ is linear then $G^*(X)$ is sub-Gaussian with $\sigma^2(G^*(X)) \leq rd^{-1}\Tr\bK_Z \sigma^2({\bK_X^{-1/2}X}).$
	\end{lemma}
	\begin{proof}
        Assume that $g^2=\EE\l[\|G^*(X)\|_2^2\r]>0.$ If assumption \eqref{eq:assump_lin} holds then for any $\alpha\in[0,1]:\,\alpha G^*\in \cG.$ Consider $\tilde{G}^*(X) = G^*(X)/g.$
        By optimality of $G^*$ for the optimal coupling $\pi^*:$
        \begin{align*}
            g&= \arg\min_{\alpha\in\l[0,g\r]}W_{2,\lambda}^2(\PP_{\alpha\tilde{G}^*(X)}, \PP_Z)\\
            &=\arg\min_{\alpha\in\l[0,g\r]}\EE_{X,Z\sim \pi^*}
            \l[\|\alpha\tilde{G}^*(X)-Z\|_2^2\right]+\lambda I(\tilde{G}(X);Z)\\
            &=\arg\min_{\alpha\in\l[0,g\r]}\alpha^2+\EE\l[\|Z\|^2\r]-2\alpha\EE_{X,Z\sim \pi^*}\l[\tilde{G}^*(X)^TZ\r]+\lambda I(\tilde{G}(X);Z)
        \end{align*}
        The above problem is minimization of a quadratic function thus $$g = \alpha^*=\min\l\{g,\EE_{X,Z\sim \pi^*}\l[\tilde{G}^*(X)^TZ\r]\r\}\leq \sqrt{\Tr{\bK_{Z}}},$$ so $\EE\l[\|G^*(X)\|_2^2\r]\leq \Tr\bK_{Z}.$ For a linear $G^*:$
        $\EE\l[\|G^*X\|_2^2\r] = \Tr {G^*}^TG^*K_X = \|G^*K_X^{1/2}\|_F^2\leq \Tr \bK_Z$ for $\tau^2 = \frac{\Tr\bK_Zr}{d}\sigma^2\l({\bK_X^{-1/2}X}\r)$
        \begin{align*}
           \EE\,e^{\frac{\|G^*X\|_2^2}{2d\tau^2}}& =\EE\, e^{\frac{\|G^*\bK_X^{1/2}\bK_X^{-1/2}X\|_2^2}{2d\tau^2}}\leq \EE\, e^{\frac{\Tr{\bK_{Z}}\|\bK_X^{-1/2}X\|_2^2}{2d\tau^2}}= \EE\, e^{\frac{\|\bK_X^{-1/2}X\|_2^2}{2r\sigma^2\l(\bK_X^{-1/2}X\r)}}\leq 2
        \end{align*}
    \end{proof}
	
	\begin{lemma}\label{lemma:cov_bound}
		For a sub-gaussian $Z\in \RR^d$ the covariance matrix trace is bounded as 
		$\Tr\bK_{Z}\leq 2d\sigma^2(Z).
		$
	\end{lemma}
	\begin{proof}
	    \begin{align*}
	        \ln 2\geq \ln\EE\,e^{\frac{\|Z\|_2^2}{2d\sigma^2(Z)}}
	        \geq \ln e^{\EE\frac{\|Z\|_2^2}{2d\sigma^2(Z)}}
	        ={\Tr{K_Z}}/\l({2d\sigma^2(Z)}\r).
	    \end{align*}
	    The first inequality follows from $Z$ being sub-Gaussian and the second one is Jensen's inequality.
	\end{proof}
	
	\begin{proof}[Proof of Theorem \ref{thm:generalization}]
		The proof is based on  \citep[Theorem~2]{mena2019statistical}. Denote $C_{d,i}$ constants depending on the dimension $d$ as we are not aiming to find the exact dependence of the bound from the dimension. Let $\lambda = 2,$ we will generalize to $\lambda\neq 2$ exacly as we did in the proof of Proposition \ref{thm:mena_res}.
		First, we rewrite $d_{\lambda}(G^*,G_n)=W_{2,\lambda}^2(\PP_{G_n(X)},\PP_Y)-W_{2,\lambda}^2(\PP_{G^*(X)},\PP_Y)$ to fit Proposition \ref{thm:mena_res}:
		\begin{align}
		d_{\lambda}&(G^*,G_n)= \l(W_{2,\lambda}^2(\PP_{G^*(X)},\QQ_Y^n)- W_{2,\lambda}^2(\PP_{G^*(X)},\PP_Y)\r)\nonumber\\
		&\quad+\l(W_{2,\lambda}^2(\PP_{G_n(X)},\PP_Y)- W_{2,\lambda}^2(\PP_{G^*(X)},\QQ_Y^n)\r)\nonumber\\
		&\leq \l(W_{2,\lambda}^2(\PP_{G^*(X)},\QQ_Y^n)- W_{2,\lambda}^2(\PP_{G^*(X)},\PP_Y)\r)\nonumber\\
		&\quad+\l(W_{2,\lambda}^2(\PP_{G_n(X)},\PP_Y)- W_{2,\lambda}^2(\PP_{G_n(X)},\QQ_Y^n)\r)\label{eq:risk}
		\end{align}
		Let $\nu^2 =\max\{2r\sigma^2(K_X^{-1/2}X)\sigma^2(Y),\sigma^2(Y)\}\leq 2r\tau^2.$ Then  
		\begin{align*}
		\sigma^2\l(G^*(X)\r)&\leq rd^{-1}\Tr\bK_Y\sigma^2\bigl(\bK_X^{-1/2}X\bigr)\\
		&\leq
		2r\sigma^2\bigl(\bK_X^{-1/2}X\bigr)\sigma^2(Y)\leq \nu^2,
		\end{align*}
		with the inequalities following from~Lemmas \ref{lemma:G_bounded}, \ref{lemma:cov_bound} and the definition of $\nu^2.$ By Proposition \ref{thm:mena_res}  applied to the expectation of the first difference in \eqref{eq:risk}:
		\begin{align}
		\EE&\l[\l\vert W_{2,\lambda}^2(\PP_{G^*(X)},\QQ_Y^n)- W_{2,\lambda}^2(\PP_{G^*(X)},\PP_Y)\r\vert\r]\nonumber\\
		&\leq C_{d,2}n^{-1/2}\bigl(1+\l(\nu^2\r)^{\l\lceil 5d/4\r\rceil+3}\bigr),\label{eq:res_easy}
		\end{align} 
		As $G_n$ depends on the sample, the proposition cannot be applied directly to the second difference, but by 
		Proposition \ref{prop:functioonalbound} for $\tsig^2 = \max\bigl\{\sigma^2_{\hat\cY}\l(G_n(X)\r), \sigma^2_{\hat\cY}(\hY),\sigma^2(Y)\bigr\}$ and $s = \l\lceil d/2\r\rceil +1:$
		\begin{align}
		&W_{2,\lambda}^2(\PP_{G_n(X)},\PP_Y)- W_{2,\lambda}^2(\PP_{G_n(X)},\QQ_Y^n)\nonumber\\
		&\leq4\l(1+\tsig^{3s}\r)\|\PP_Y-\QQ_Y^n\|_{\cF^s},\label{eq:mena_res}.
		\end{align}
		Note that $\cF^s$ only depends on $s$ and not on the sub-gaussian parameters of $Y$ and $GX.$
		\noindent 
		Taking expectation over the sample in \eqref{eq:mena_res} we get:
		\begin{align}
		\bigl(\EE&\big[W_{2,\lambda}^2(\PP_{G_n(X)},\PP_Y)- W_{2,\lambda}^2(\PP_{G_n(X)},\QQ_Y^n)\big]\bigr)^2\nonumber\\
		&\leq8\EE\l[1+\tsig^{6s}\r]\EE\|\PP_Y-\QQ_Y^n\|_{\cF^s}^2\nonumber\\
		&\leq\l(1+\sigma^2(Y)^{d+2}\r)n^{-1}C_{d,3}\EE\l[1+\tsig^{6s}\r]\label{eq:sq_bound0}\\
		&\leq\l(1+\nu^{2d+4}\r)n^{-1}C_{d,3}\EE\l[1+\tsig^{6s}\r],\label{eq:sq_bound}
		\end{align}
		where \eqref{eq:sq_bound0} follows from Proposition \ref{prop:coveringbound} and \eqref{eq:sq_bound} from the definition of $\nu.$
		By Lemma \ref{lemma:cov_bound} we have $\Tr\bK_{\hY}\leq 2d\sigma^2(\hY),$ so
		\begin{align}
		\sigma^2_{\hat\cY}&(G_n(X))\leq  d^{-1}\Tr\bK_{\hY}r\sigma^2\bigl({\bK_X^{-1/2}X}\bigr)\nonumber\\
		&\leq 2r\sigma^2\bigl({\bK_X^{-1/2}X}\bigr)\sigma^2_{\hat\cY}\bigl(\hY\bigr)\leq \sigma^2_{\hat\cY}(\hY)\nu^2/\sigma^2(Y),\label{ineq:yu01}
		\end{align}
		where the first inequality follows from Lemma~\ref{lemma:G_bounded} and the second one from Lemma~\ref{lemma:cov_bound}.
		Taking expectation of $\tsig^{6s}:$
		\begin{align} &\EE\bigl[\tsig^{6s}\bigr]=\EE\bigl[\max\bigl\{\sigma^2_{\hat\cY}\bigl(\hY\bigr),\sigma^2(Y), \sigma_{\hat\cY}^2\l(G_n(X)\r)\bigr\}^{3s}\bigr]\nonumber\\
		&\leq\nu^{6s}\EE\biggl[\max\bigl\{1, \sigma^2_{\hat\cY}\bigl(\hY\bigr)/\sigma^2(Y)\bigr\}^{3s}\biggr]\leq 2(3s)^{3s}\nu^{6s}\label{eq:sig_bound},
		\end{align}
		where \eqref{eq:sig_bound} is due to Proposition~\ref{prop:sigmamoments}; plugging \eqref{eq:sig_bound} in \eqref{eq:sq_bound} gives\!
		\begin{align}
		\EE&\big[W_{2,\lambda}^2(\PP_{G_n(X)},\PP_Y)- W_{2,\lambda}^2(\PP_{G_n(X)},\QQ_Y^n)\big]\nonumber\\
		&\leq \sqrt{2(1+2(3s)^{3s}\nu^{6s})C_{d,3}n^{-1}\l(1+\nu^{d+2}\r)}\nonumber\\
		&\leq C_{d,4}n^{-1/2}\bigl(1+\l(\nu^2\r)^{\l\lceil 5d/4\r\rceil+3}\bigr)\label{eq:th1_res}
		\end{align}
		Combining \eqref{eq:th1_res} and \eqref{eq:res_easy} we get for $\lambda = 2:$
		\begin{align*}
		\EE\l[d_{\lambda}(G^*,G_n)\r] &\leq C_{d,5} n^{-1/2}(1+(\nu^2)^{\l\lceil 5d/4\r\rceil+3})\\
		&\leq K_d n^{-1/2}(1+(\tau^2)^{\l\lceil 5d/4\r\rceil+3}),
		\end{align*}
		Consider $\lambda\neq 2.$ Then for any $\lambda>0:$
		\begin{align*}
		W_{2,2}^2&(\PP_{Z\sqrt{2/\lambda}},\PP_{Y\sqrt{2/\lambda}}) \nonumber\\
		&= \inf_{\pi\in\Pi\l((\PP_{Z},\PP_{Y})\r)}2\EE\l[\|Z -Y\|^2\r]/\lambda+2 I(Z; Y) \\
		&= 2 W_{2,\lambda}^2(\PP_{Z},\PP_Y)/\lambda
		\end{align*}
		Thus, noting that for a sub-gaussian $Z:$
		\begin{align*}
		\EE \exp \biggl(\frac{\|Z\sqrt{2/\lambda}\|_2^2}{2r\sigma_Z^22/\lambda}\biggr) = \EE \exp \biggl(\frac{\|Z\|_2^2}{2r\sigma_Z^2}\biggr)\leq 2
		\end{align*}
		we conclude that $\sigma^2(Z\sqrt{\lambda/2})=2\sigma^2(Z)/\lambda.$ Plugging the result into the bound \eqref{eq:th1_res}  we get
		\begin{align}
		\EE\l[d_{\lambda}(G^*,G_n)\r]
		&\leq K_d \lambda n^{-1/2}\bigl(1+(2\tau^2/\lambda)^{\l\lceil 5d/4\r\rceil+3}\bigr)/2\label{eq:thm_fin_res}.
		\end{align}
	\end{proof}
	
	\begin{proof}[Proof of Theorem \ref{thm:generalization_lip}]
		The proof follows the same path as the proof of Theorem \ref{thm:generalization} with the only difference being in bounding the sub-Gaussian parameters. 
		
		For $G\in\cG$ let $\ciG(X) = G(X) - G(0)$ -- a shifted function. Note that $\ciG$ need not be in $\cG.$ 
		To avoid confusion we let $H^* = \argmin_{H\in\cG}S_{\lambda}(P_{H(X)}, P_Y)$ and $H_n = \argmin_{H\in\cG}S_{\lambda}(P_{H(X)}, Q^n_Y)$ -- the population and empirical solutions to Sinkhorn W2GANs. Since 
		\begin{align*}
		\EE\l[\|G(X) - Y\|^2\r] &= 
		\EE\l[\|\ciG(X) - Y + G(0)\|^2\r]\\
		& = \EE\l[\|\ciG(X) - Y\|^2 \r]+ 2G(0)^T\EE[\ciG(X)-Y]+ \|G(0)\|^2
		\end{align*}
		and $I(G(X), Y) = I(\ciG(X), Y),$ entropy-regularized Wasserstein distance decomposes as $W^2_{2,\lambda}(P_{G(X)}, P_Y) = W^2_{2,\lambda}(P_{\ciG(X)}, P_Y) + 2G(0)^T\EE[\ciG(X)-Y]+ \|G(0)\|^2$ 
		As in the proof of Theorem \ref{thm:generalization} we decompose the generalization error:
		\begin{align}
		d_{\lambda}(G^*,G_n)&= W_{2,\lambda}^2(\PP_{G_n(X)},\PP_Y)- W_{2,\lambda}^2(\PP_{G^*(X)},\PP_Y)\nonumber\\
		&\leq W_{2,\lambda}^2(\PP_{G^*(X)},\QQ_Y^n)- W_{2,\lambda}^2(\PP_{G^*(X)},\PP_Y)\nonumber\\
		&\qquad+ W_{2,\lambda}^2(\PP_{G_n(X)},\PP_Y)- W_{2,\lambda}^2(\PP_{G_n(X)},\QQ_Y^n)\nonumber\\
		&= \l(W_{2,\lambda}^2(\PP_{\ciG^*(X)},\QQ_Y^n)- W_{2,\lambda}^2(\PP_{\ciG^*(X)},\PP_Y)\r)\nonumber\\
		&\qquad+\l(W_{2,\lambda}^2(\PP_{\ciG_n(X)},\PP_Y)- W_{2,\lambda}^2(\PP_{\ciG_n(X)},\QQ_Y^n)\r)\nonumber\\
		&\qquad +2\l(G^*(0) - G_n(0)\r)^T\l(\EE Y - \EE_{\hat\cY} \hY\r),\label{eq:d_extra_term}
		\end{align}
		where the last inequality follows as \eqref{eq:risk}.
		Similarly, for Sinkhorn W2GAN the generalization error is:
		\begin{align}
		d^S_{\lambda}(H^*,H_n)&=
		S_{\lambda}(\PP_{H_n(X)},\PP_Y) - S_{\lambda}(\PP_{H^*(X)},\PP_Y)\nonumber\\
		&\leq \l(S_{\lambda}(\PP_{H^*(X)},\QQ_Y^n)- S_{\lambda}(\PP_{H^*(X)},\PP_Y)\r)\nonumber\\
		&\qquad+\l(S_{\lambda}(\PP_{H_n(X)},\PP_Y)- S_{\lambda}(\PP_{H_n(X)},\QQ_Y^n)\r)\nonumber\\
		&= \l(W_{2,\lambda}^2(\PP_{H^*(X)},\QQ_Y^n)- W_{2,\lambda}^2(\PP_{H^*(X)},\PP_Y)\r)\nonumber\\
		&\qquad+\l(W_{2,\lambda}^2(\PP_{H_n(X)},\PP_Y)- W_{2,\lambda}^2(\PP_{H_n(X)},\QQ_Y^n)\r)\nonumber\\
		&=\l(W_{2, \lambda}^2(\PP_{\ciH^*(X)},\QQ_Y^n)- W_{2, \lambda}^2(\PP_{\ciH^*(X)},\PP_Y)\r)\nonumber\\
		&\qquad+\l(W_{2, \lambda}^2(\PP_{\ciH_n(X)},\PP_Y)- W_{2, \lambda}^2(\PP_{\ciH_n(X)},\QQ_Y^n)\r)\nonumber\\
		&\qquad+2\l(H^*(0) - H_n(0)\r)^T\l(\EE Y - \EE_{\hat\cY} \hY\r)\label{eq:d_S_bound}
		\end{align}
		The RHS of \eqref{eq:d_S_bound} and \eqref{eq:d_extra_term} are the same as the RHS of \eqref{eq:risk}. Note that for any $G\in\cG$ and for $\sigma^2 = \sigma^2(X)rL^2/d$ by the $L$-Lipschitzness of $G:$
        \begin{align*}
           \EE\,e^{\frac{\|\ciG(X)\|_2^2}{2d\sigma^2}}& =\EE\, e^{\frac{\|G(X)-G(0)\|_2^2}{2d\sigma^2}}\leq \EE\, e^{\frac{L^2\|X\|_2^2}{2d\sigma^2}}= \EE\, e^{\frac{\|X\|_2^2}{2r\sigma^2(X)}}\leq 2,
        \end{align*}
		thus $\ciG(X)$ and $\ciH(X)$ are both sub-Gaussian,  $\max\{\sigma^2(\ciG(X)),\sigma^2(\ciH(X))\}\leq  \sigma^2(X)rL^2/d\leq \tau^2.$ 
		
		The next part of the proof follows the proof of Theorem \ref{thm:generalization} with $\nu^2 = \tau^2,$ and $\ciG_n$ and $\ciH_n$ in place of $G_n$, $\ciG^*$ and $\ciH^*$ in place of $G^*$ for the entropic and Sinkhorn W2GAN cases respectively. 
		Indeed, for entropic W2GAN \cref{eq:res_easy,eq:mena_res,eq:sq_bound} only require that $\max\{\sigma^2(\ciG_n(X)),\sigma^2(\ciG^*(X)),\sigma^2(Y)\}\leq \nu^2.$ 
		As $\sigma^2_{\hat\cY}(\ciG_n(X))\leq L^2\sigma^2(X),$ in place of \eqref{eq:sig_bound} we have for 
		$\tsig^2 = \max\bigl\{\sigma^2_{\hat\cY}\l(\ciG_n(X)\r), \sigma^2_{\hat\cY}(\hY),\sigma^2(Y)\bigr\}$ and $s = \l\lceil d/2\r\rceil +1:$ 
    	\begin{align*}
    	    \EE\bigl[\tsig^{6s}\bigr]&\leq\EE\bigl[\max\bigl\{\sigma^2_{\hat\cY}\bigl(\hY\bigr),\sigma^2(Y), \sigma_{\hat\cY}^2\l(G_n(X)\r)\bigr\}^{3s}\bigr]\nonumber\\
    	    &\leq\EE\bigl[\max\bigl\{\sigma^2_{\hat\cY}\bigl(\hY\bigr),\sigma^2(Y), L^2\sigma^2(X)\bigr\}^{3s}\bigr]\nonumber\\
    		&\leq\nu^{6s}\EE\biggl[\max\bigl\{1, \sigma^2_{\hat\cY}\bigl(\hY\bigr)/\sigma^2(Y)\bigr\}^{3s}\biggr]\leq 2(3s)^{3s}\nu^{6s},
	    \end{align*}
	    where the last inequality is due to Proposition~\ref{prop:sigmamoments}. So, eq.  \eqref{eq:th1_res} and \eqref{eq:thm_fin_res} hold, i.e.
		\begin{align}
		\EE d_{\lambda}(G^*,G_n) 
		&\leq K_d \lambda n^{-1/2}\bigl(1+(2\tau^2/\lambda)^{\l\lceil 5d/4\r\rceil+3}\bigr)/2 +2\EE\l[\l(G^*(0) - G_n(0)\r)^T\l(\EE Y - \EE_{\hat\cY} \hY\r)\r]\nonumber\\
		&\leq K_d \lambda n^{-1/2}\bigl(1+(2\tau^2/\lambda)^{\l\lceil 5d/4\r\rceil+3}\bigr)/2 +2\sqrt{\EE\|G_n(0)\|_2^2}\sqrt{\Tr K_Y/n},\label{eq:ciG_risk_bound}
		\end{align}
		where the last inequality follows from the independence of $G^*(0)$ and the sample and the Cauchy-Schwarz inequality. 
		
		For Sinkhorn W2GAN the above results in
		\begin{align}
		\EE d_{\lambda}^S(H^*,H_n)
		&\leq K_d \lambda n^{-1/2}\bigl(1+(2\tau^2/\lambda)^{\l\lceil 5d/4\r\rceil+3}\bigr)/2+2\sqrt{\EE\|H_n(0)\|_2^2}\sqrt{\Tr K_Y/n}\label{eq:ciH_risk_bound}
		\end{align} 
		
		We will now bound the last term of \eqref{eq:ciG_risk_bound} via Lemma \ref{lemma:G_bounded} and Lipschitzness of $G:$
		\begin{align*}
		    \EE\|G_n(0)\|_2^2&\leq 2\EE\|G_n(X) - G_n(0)\|_2^2 + 2\EE\|G_n(X)\|_2^2 \leq2L^2\Tr{K_X} + 2\Tr{K_Y}\leq 8d\tau^2,
		\end{align*}
		where the last inequality follows from Lemma \ref{lemma:cov_bound}.
		\eqref{eq:ciG_risk_bound} thus becomes:
		\begin{align}
		\EE d_{\lambda}(G^*,G_n) 
		&\leq K_d \lambda n^{-1/2}\bigl(1+(2\tau^2/\lambda)^{\l\lceil 5d/4\r\rceil+3}\bigr)/2 +2\sqrt{8d\tau^2\Tr{K_Y}/{n}}\nonumber\\
		&\leq K_d \lambda n^{-1/2}\bigl(1+(2\tau^2/\lambda)^{\l\lceil 5d/4\r\rceil+3}\bigr)/2+ 8d\tau^2/\sqrt{n}\label{eq:cauch_schw},
		\end{align}
		where the last inequality fllows from Lemma \ref{lemma:cov_bound}.
		Redefining $K_d$ completes the proof for Entropic W2GAN generalization error. 
		
		To complete the proof we need to bound $\EE \|H_n(0)\|_2^2.$ 
		We first note that $0\in\cG,$ so by optimality of $H_n:$ 
		\begin{align}
		    S_{\lambda}(P_{H_n(X)},Q^n_Y)&\leq S_{\lambda}(\delta_0,Q^n_Y)= \EE_{\hat\cY}\|\hY\|_2^2  - W_{2,\lambda}^2(Q^n_Y, Q^n_Y)/2\leq \EE_{\hat\cY}\|\hY\|_2^2,\label{eq:S_upper}
		\end{align}
		where $\delta_0$ is a point mass at $0.$
		As $ S_{\lambda}(P_{\ciH_n(X)}, Q^n_Y)\geq 0:$
		\begin{align}
		    S_{\lambda}(P_{H_n(X)}, Q^n_Y)& =
		    S_{\lambda}(P_{\ciH_n(X)}, Q^n_Y)+ \|H(0)\|_2^2 + 2H(0)^T\EE_{\hat\cY}\l[\hY - \ciH_n(X)\r]\nonumber\\
		    &\geq \|H_n(0)\|_2^2 + 2H_n(0)^T\EE_{\hat\cY}\l[ \hY-\ciH_n(X)\r]\label{eq:S_lower}
		\end{align}
	    Combining \eqref{eq:S_upper} and \eqref{eq:S_lower} and taking the expectation over the sample $\hat\cY$ we get:
	    \begin{align}
	        2d\tau^2&\geq \Tr K_Y=\EE\EE_{\hat\cY}\|\hY\|_2^2 \nonumber\\
	        &\geq\EE\|H_n(0)\|_2^2 - 2\EE\l[
	        \|H_n(0)\|\EE_{\hat\cY}\l[\|\hY\| + L\|X\|\r]\r]\nonumber\\
	        &\geq \EE\|H_n(0)\|_2^2 -2\sqrt{\EE
	        \|H_n(0)\|_2^2}(\sqrt{\Tr K_Y} + L\sqrt{\Tr K_X})\nonumber\\
	        &\geq \EE\|H_n(0)\|_2^2-4\sqrt{\EE
	        \|H_n(0)\|_2^2}\sqrt{2d\tau^2}\nonumber
	    \end{align}
	    The above inequality implies that  $\EE\|H_n(0)\|_2^2\leq 40d\tau^2.$ From \eqref{eq:ciH_risk_bound}:
		\begin{align*}
		\EE d_{\lambda}^S(H^*,H_n)
		&\leq K_d \lambda n^{-1/2}\bigl(1+(2\tau^2/\lambda)^{\l\lceil 5d/4\r\rceil+3}\bigr)/2 +2\sqrt{\EE\|H_n(0)\|_2^2}\sqrt{\Tr K_Y/n}\\
		&\leq K_d \lambda n^{-1/2}\bigl(1+(2\tau^2/\lambda)^{\l\lceil 5d/4\r\rceil+3}\bigr)/2 +2\sqrt{40d\tau^2}\sqrt{2d\tau^2/n}\\
		&\leq K_d \lambda n^{-1/2}\bigl(1+(2\tau^2/\lambda)^{\l\lceil 5d/4\r\rceil+3}\bigr)/2 +20d\tau^2/\sqrt{n}
		\end{align*} 
		Redefining $K_d$ completes the proof.
	\end{proof}
	
	\section{Computational Convergence}
	For the sake of completeness, in this section we discuss some results on the computational convergence of entropic optimal transport and Sinkhorn divergence and emphasize the advantages of these regularization methods from an optimization perspective. A more detailed discussion can be found in \citep{sanjabi2018convergence,feydy2019interpolating}. As was previously mentioned, entropic regularization makes the problem strongly convex, which in turn facilitates convergence.  Note that since the optimal solution to \eqref{eq:OTreg} is known to satisfy \eqref{eq:SK1},\eqref{eq:SK0}, Sinkhorn-Knopp iterates \eqref{eq:SK_it},\eqref{eq:SK_it1} or any other method can be used to solve the inner problem close to optimality. In contrast, computing the unregularized optimal transport requires the use of linear programming techniques which are computationally infeasible in many machine learning applications. 
	
	Moreover, \citep[Theorem 3.1]{sanjabi2018convergence} shows that  under mild conditions on the generator set $\cG$ and the distributions of $P_Y, P_X$, entropy-regularized Wasserstein distance is Lipschitz smooth, i.e. has a Lipschitz continuous gradient with respect to the the parameters of the generator. If we let the generator set $\cG$ be parametrized by $\theta\in\Theta,$ i.e. $\cG = \{G_{\theta}\mid \theta\in\Theta\})$ then
	$$|\grad_{\theta}W_{2,\lambda}^2(P_{G_{\theta_1}(X)}, P_Y) - \grad_{\theta}W_{2,\lambda}^2(P_{G_{\theta_2}(X)}, P_Y)|\leq L\|\theta_1 - \theta_2\|,$$
	where $L$ is a constant depending on $P_X, P_Y, \cG$ and $\lambda.$ Moreover, the optimal coupling $\pi^*(\theta)$ is a Lipschitz continuous function of $\theta:$
	\[\|\pi^*(\theta_1) - \pi^*(\theta_2)\|_1\leq
	\frac{L_0}{\lambda}\|\theta_1 - \theta_2\|\]
	The above indicates that small changes in the generator parameter $\theta$ result in small changes in the optimal coupling. Therefore, after the gradient step on the generator parameters $\theta,$ finding the regularized Wasserstein distance is easier since the discriminator parameters from the previous step are close to the optimal ones for the current step, while the Lipschitz smoothness of regularized Wasserstein distance in $\theta$ results in faster convergence of optimization.
	
	Note that first-order optimization methods commonly used for neural network optimization require calculating the gradients $\grad_{\theta}W_{2,\lambda}^2(P_{G_{\theta}(X)}, P_Y),$ which requires knowing the optimal dual potentials, but since they are found numerically, they can only be computed up to some positive accuracy, so the smoothness of the gradient of the entropic Wasserstein distance with respect to the accuracy up to which the dual potentials are calculated plays a crucial role in the convergence of the optimization. 
	More precisely, if the inner problem of calculating Entropic Wasserstein distance is solved up to a certain accuracy $\epsilon,$ it can be shown that the gradient step on the outer problem of finding $G$ is $O(\sqrt{\epsilon/\lambda})$- close to optimal, which makes training stable (see \citep[Theorem 4.1]{sanjabi2018convergence}). In contrast, the training of W2GAN, i.e. based on the unregularized squared Wasserstein distance, is known to be unstable even for the linear generator $G$ and quadratic discriminator  \citep{feizi2017understanding} when $r<d,$ and the training methods for  Wasserstein GAN \citep{arjovsky2017wasserstein,gulrajani2017improved} do not converge locally  with simultaneous or alternating gradient descent \citep{mescheder2018training}. 
	
	Finally, we note the following optimization convergence result from \citep[Theorem 4.2.]{sanjabi2018convergence}. Under mild conditions on $\cG, P_X,P_Y$  it can be shown that when stochastic gradient is used to solve 
	$$\min_{\theta} f(\theta) = \min_{\theta} W_{2,\lambda}^2(P_{G_{\theta}(X)}, P_Y),$$
	where $G(\cdot)$ is parametrized by $\theta$,
	the random iterates $\theta_1,\dots,\theta_T$ satisfy
	$$\min_{t=1\ldots T}\EE[\|\nabla f(\theta_t)\|^2]\leq O(1/\sqrt{T}) + O(\eps/\lambda).$$
Here constants in $O(\cdot)$ depend on the class of generators $G\in\cG$ and the distributions $P_X,P_Y,$ and $T$ is the number of iterations of stochastic gradient descent and $\eps$ is the precision,  to which the inner problem is solved.  The expectation is over the randomness in the algorithm. The theorem implies that if there are enough iterations to get the discriminator close to optimality the training reaches a stable point of small $\EE[\|\nabla f(\theta_t)\|^2].$ Since Sinkhorn divergence is a linear combination of entropy-regularized Wasserstein distances, a similar result holds for it and the optimization is stable.
 	
	\section{Experiments}
	In our experiments we aim to contrast and compare the performance of Sinkhorn GAN (label: SGAN) and 1-Wasserstein GAN WGAN (label: WGAN) for linear generators. Entropic W2GAN is omitted from the comparison due to the fact that it leads to a biased solution as shown in Theorem~\ref{thm:PCA}. Following the experimental evaluations in \citep{feizi2017understanding}, we generate $n=10^5$ samples from a $d=32$ dimensional Gaussian
    distribution $\cN(0, K)$ where $K$ is a random positive semi-definite matrix normalized to have Frobenius norm 1. We train WGAN with weight clipping (label: WGAN-WC) \citep{arjovsky2017wasserstein} and WGAN with gradient penalty (label: WGAN-GP) \citep{gulrajani2017improved} -- two common methods to ensure Lipschitzness of the discriminators. We use the linear generator and a neural network discriminator
    with hyper-parameter settings as recommended in \citep{gulrajani2017improved}. The discriminator neural network has three hidden
    layers, each with 64 neurons and ReLU activation functions. 
    
	The pseudocode of our optimization for Sinkhorn GAN can be found in Algorithm \ref{alg:SGD}. 
	The algorithm is similar to  \citep{sanjabi2018convergence}, where 
	we assume that the generators are parametrized by $\theta,$ i.e. $G(X) = G_{\theta}(x)$ and we apply stochastic gradient descent on $\theta$. Note that at every step of the gradient descent algorithm we need to calculate the gradient of the Sinkhorn divergence, $\nabla_\theta S_{\lambda}(P_{G_\theta(X)}, P_Y)$. 
	From the definition of the Sinkhorn divergence in \eqref{eq:OTsink}, to compute $\nabla_\theta S_{\lambda}(P_{G_\theta(X)}, P_Y)$ we need to compute  $\nabla_\theta W_{2, \lambda}^2(P_{G_\theta(X)}, P_{G_\theta(X)})$ and $\nabla_\theta W_{2, \lambda}^2(P_{G_\theta(X)}, P_Y)$. (Note that the third term $W_{2,\lambda}^2(P_Y,P_Y)$ is irrelevant since it doesn't depend on the generator.) From \eqref{eq:OTreg_dual}, we have the following dual representation:
	\begin{align}
    W_{2,\lambda}^2(\PP_{G_\theta(X)},\PP_Y)&=\sup_{\substack{\psi\in L_{\infty}(\cY)\\ \phi\in L_{\infty}(G_\theta(\cX))}} \EE\left[\psi(Y) + \phi(G_\theta(X))\right] +\lambda\nonumber\\
    &\qquad-\lambda\EE_{(X,Y)\sim P_X\times P_Y}\l[e^{\frac{\phi(G_\theta(X))+\psi(Y)-\|G_\theta(X)-Y\|_2^2}{\lambda}}\r]\label{eq:OTreg_dual1}\\
    W_{2,\lambda}^2(\PP_{G_\theta(X)},\PP_{G_\theta(X)})&=\sup_{\phi^x\in L_{\infty}(G_\theta(\cX))} 2\EE\left[\phi^x(G_\theta(X))\right]+\lambda \nonumber\\
    &\qquad-\lambda\EE_{(X_1,X_2)\sim P_X\times P_X}e^{\frac{\phi^x(G_\theta(X_1))+\phi^x(G_\theta(X_2))-\|G_\theta(X_1)-G_\theta(X_2)\|_2^2}{\lambda}},\label{eq:OTreg_dualxx}
	\end{align}
	Now assume that we have access to approximations of the optimal dual potentials for \eqref{eq:OTreg_dual1}, which for simplicity we also denote by  $\phi$ and $\psi$. Then by using \eqref{eq:opt_coup}, we can obtain an approximation of the optimal coupling given by
	\begin{align}
	    \pi(G_\theta(x), y)& = P_{G_\theta(X)}(G_\theta(x))P_Y(y)e^{\frac{\phi(G_\theta(x))+\psi(y)-\|G_\theta(x)-y\|_2^2}{\lambda}}\nonumber\\
	    &= P_{G_\theta(X)}(G_\theta(x))P_Y(y)\mu(x, y)\label{eq:mu}.
	\end{align}
	Using the above in place of the coupling in the primal formulation of the entropic 2-Wasserstein distance we can then compute an approximation of the desired gradient 
	\begin{align*}
	    \nabla_\theta W_{2, \lambda}^2(P_{G_\theta(X)}, P_Y)
	    \approx\EE_{(X,Y)\sim P_X\times P_Y}[\mu(X,Y)\nabla_\theta (\|G_{\theta}(X) - Y\|^2)].
	\end{align*}
	Analogously, if $\phi^x(\cdot)$ is an approximate optimal dual potential for \eqref{eq:OTreg_dualxx} then
	\begin{align*}
	   \nabla_\theta W_{2, \lambda}^2(P_{G_\theta(X)}, P_{G_\theta(X)})\approx\EE_{(X_1,X_2)\sim P_X\times P_X}[\mu^x(X_1,X_2)\nabla_\theta (\|G_{\theta}(X_1) - G_{\theta}(X_2)\|^2)],
	\end{align*}
	where 
	\begin{align}
	    \mu^x(x_1,x_2) = e^{\frac{\phi^x(G_\theta(x_1))+\phi^x(G_\theta(x_2))-\|G_\theta(x_1)-G_\theta(x_2)\|_2^2}{\lambda}}\label{eq:mu_x}.
	\end{align}
	Finally, the gradient of the Sinkhorn divergence is approximated via 
	\begin{align*}
	    \nabla_\theta S_{\lambda}(P_{G_\theta(X)}, P_Y)
	    &\approx \EE_{(X,Y)\sim P_X\times P_Y}[\mu(X,Y)\nabla_\theta (\|G_{\theta}(X) - Y\|^2)]\\
	    &\qquad- \EE_{(X_1,X_2)\sim P_X\times P_X}[\mu^x(X_1,X_2)\nabla_\theta (\|G_{\theta}(X_1) - G_{\theta}(X_2)\|^2)]
	\end{align*}
	
	Since the expectations cannot be calculated exactly, we further approximate the gradient with an empirical expectation over a batch of size $S,$ which results in a mini-batch stochastic gradient descent on $S_{\lambda}(P_{G(X)}, P_Y):$ for a sample $x^1, \cdots, x^S \stackrel{i.i.d.}{\sim} P_X, y^1, \cdots, y^S \stackrel{i.i.d.}{\sim} P_{\hat Y}$ the gradient approximation is given by
	\begin{align*}
	    \nabla_\theta S_{\lambda}(P_{G_\theta(X)}, P_Y)
	    &\approx \frac{1}{S^2}\sum_{i, j=1}^S[\mu(x_i,y_j)\nabla_\theta (\|G_{\theta}(x_i) - y_j\|^2)]\\
	    &\qquad- \frac{1}{S^2}\sum_{i, j=1}^S\mu^x(x_i,x_j)\nabla_\theta (\|G_{\theta}(x_i) - G_{\theta}(x_j)\|^2)]
	\end{align*}
	
	We note that the optimal dual potentials for $W_{2,\lambda}^2(P_{G(X)}, P_{G(X)})$ for Gaussian $X$ and a linear generator can indeed be found analytically as a function of $G,$ but since it is not possible to analytically compute the potentials in the case of a more complex $G(\cdot)$ and since we do not use the linearity of $G(\cdot)$ when computing the unregularized Wasserstein distance, to give the models a fair comparison, we find $W_{2,\lambda}^2(P_{G(X)}, P_{G(X)})$ numerically.
\begin{algorithm}[t]
  \caption{SGD for GANs}
  \label{alg:SGD}
\small{
  \begin{algorithmic}
  	\STATE INPUT:  $P_X$, $P_{\hat Y},$ $\lambda$, $S$, $\theta_0$, step sizes $\{\alpha_t>0\}_{t=0}^{T-1}$
	\FOR  {$t=0, \cdots, T-1$}\STATE Sample I.I.D. points $x^1_t, \cdots, x^S_t \sim P_X, y^1_t, \cdots, y^S_t \sim P_{\hat Y}$
	\STATE Call the oracle to find $\epsilon$-approximate maximizers $(\phi_t,\psi_t), \phi^x_t$ for the dual formulations \eqref{eq:OTreg_dual1}, \eqref{eq:OTreg_dualxx}
	\STATE Compute 
	\begin{align}
	    g_t &= \frac{1}{S^2}\sum_{i,j}\bigg( \mu_t(G_\theta(x^i_t),y^j_t)\nabla_\theta (\|G_\theta(x^i_t) - y^j_t\|^2)\\
	    &\qquad- \frac{1}{2}\mu_t^x(x^i_t, x^j_t)\nabla_\theta(\|G_\theta(x^i_t)-G_\theta(x^j_t)\|^2)\bigg)\nonumber
	\end{align}
	where $\mu_t, \mu_t^x$ are computed using $(\phi_t,\psi_t)$  and $\phi_t^x$ based on \eqref{eq:mu}, \eqref{eq:mu_x}.
	\STATE Update $\theta_{t+1} \leftarrow \theta_t-\alpha_t g_t$
	\ENDFOR
  \end{algorithmic}}%
\end{algorithm}

	Note that the in the above discussion, we assumed that we have access to approximations of the optimal dual potentials. These optimal dual potentials can be computed in two different ways. The first way is to compute them via the Sinkhorn-Knopp algorithm \citep{feydy2019interpolating} 
	(label: SGAN-NP), which allows to omit the discriminator network from the GAN and compute the dual potentials in a non-parametric fashion. 
	Another way of calculating approximations of the dual potentials is to represent the dual potentials as neural networks and update them using stochastic  gradient descent  on \eqref{eq:OTreg_dual}  (label: SGAN-P). 
	
	On the one hand, using neural networks helps preserve the history of the seen examples and might help the dual potentials generalize better. On the other hand, using Sinkhorn-Knopp algorithm is more precise for computing the Sinkhorn divergence between the empirical distributions.
	We compared the two approaches and didn't find any significant differences.
	To compare WGAN and SGAN as they minimize different objectives, we evaluate their performance by calculating Frobenius distance between the covariance matrix of the generated distribution $P_{G(X)}$ and the covariance matrix of the target distribution $P_Y$  (true covariance, bottom row). We also calculate the Frobenius distance between the covariance matrix of the generated distribution $P_{G(X)}$ and the optimal covariance matrix for W2GAN (\ref{thm:PCA}) $P_{G^*(X)}$ (optimal covariance, bottom row) in Figure 1 for two values of the dimensions of the latent random variable, $r = 4$ and $r=8.$  We run the experiments for 500 epochs with a batch size of 200. In these experiments, we observe that different versions of SGAN enjoy similar behaviour and the covariance matrix of the output distribution is closer to the one of the target distribution compared to standard Wasserstein GANs. We note that the distance to the true covariance has a higher floor here, since the error cannot be zero, i.e. the $d$-dimensional Gaussian distribution cannot be approximated as a function of the $r$-dimensional one with error converging to 0).

	\begin{figure*}[t] \label{fig:fig7}
	\centering
  \begin{minipage}{0.5\linewidth}
    \includegraphics[width=\linewidth]{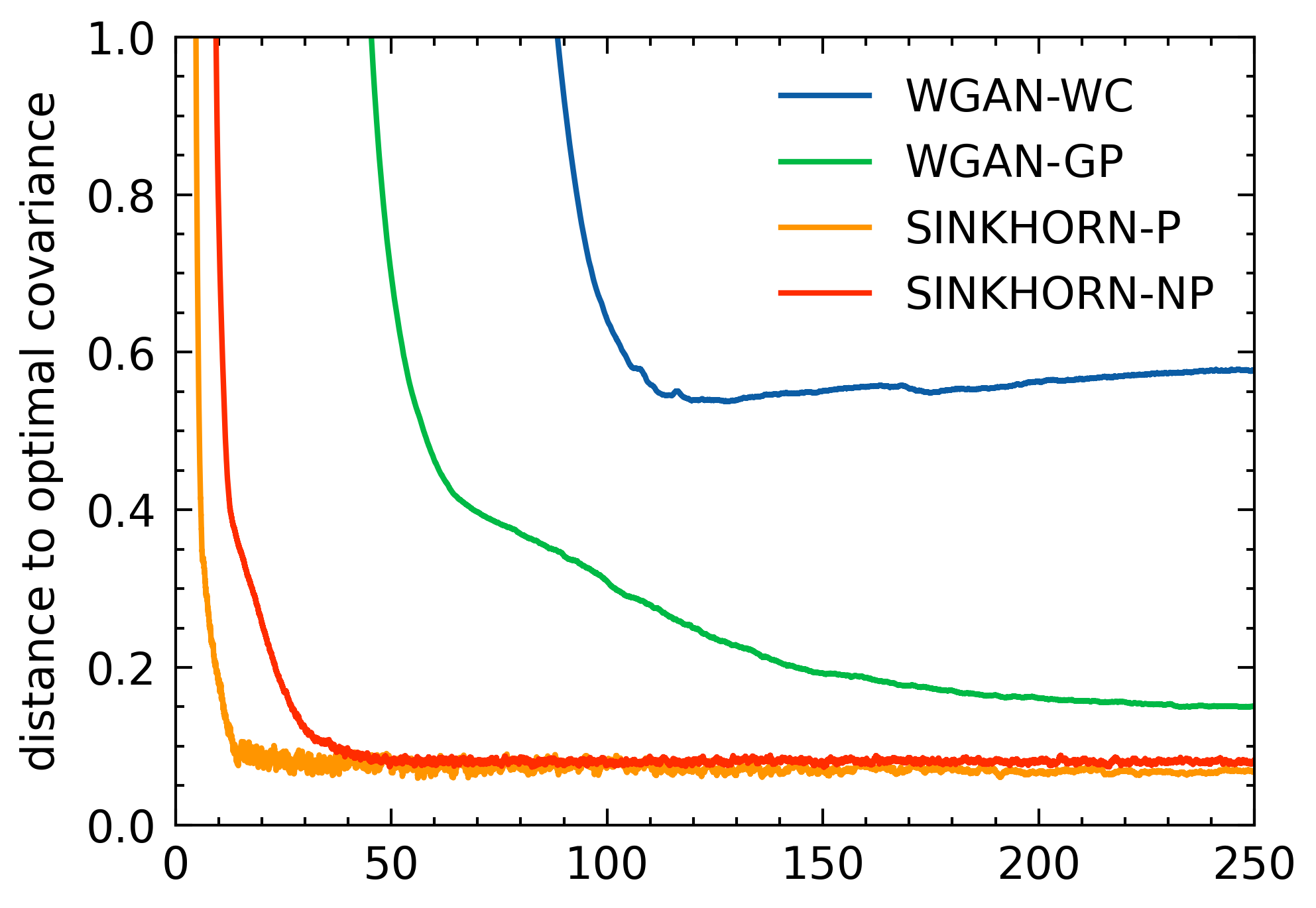}
  \end{minipage}
  \begin{minipage}{0.5\linewidth}
    \includegraphics[width=\linewidth]{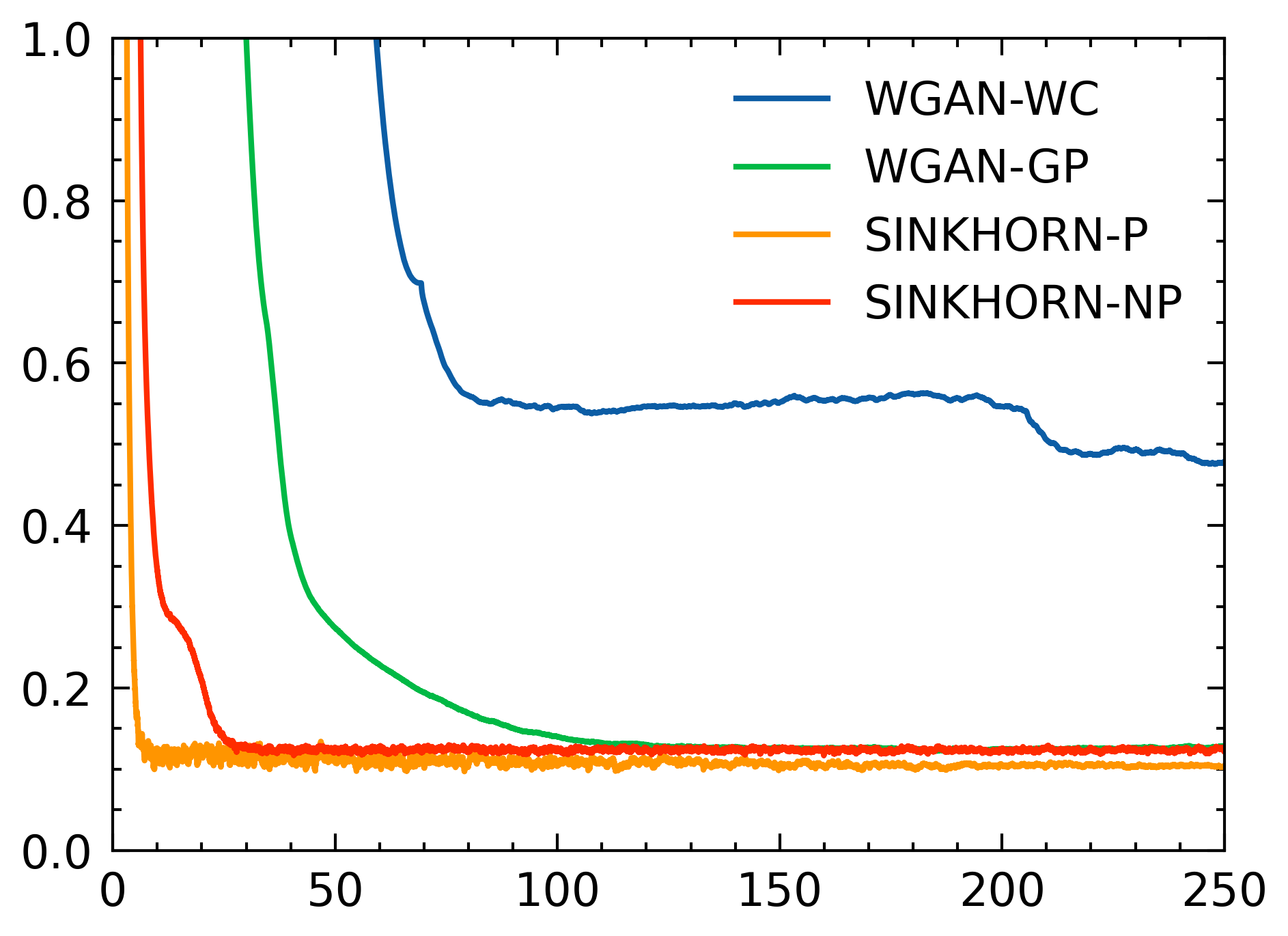} 
  \end{minipage}
  \begin{minipage}{0.5\linewidth}
    \includegraphics[width=\linewidth]{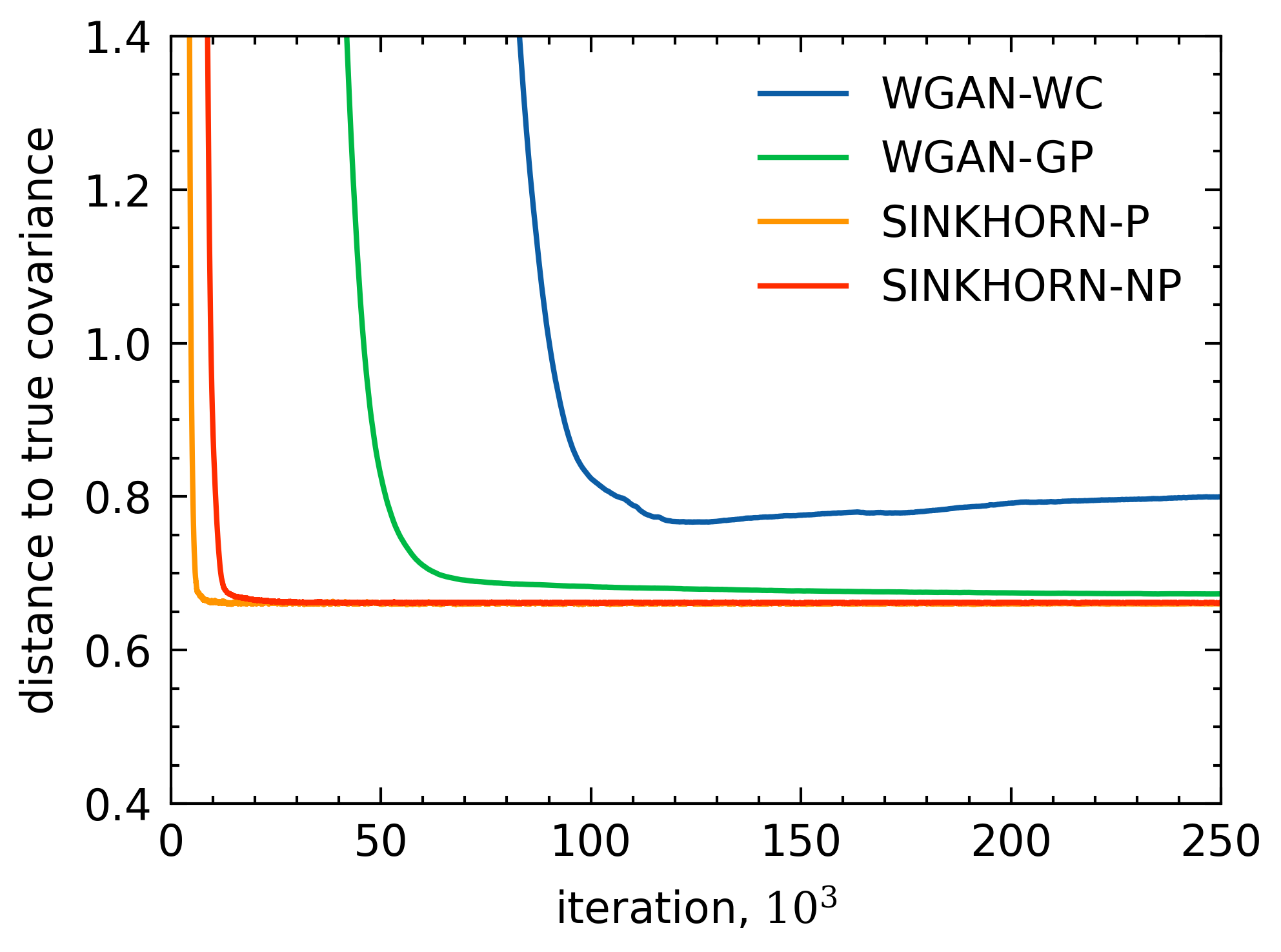} 
  \end{minipage}
  \begin{minipage}{0.5\linewidth}
    \includegraphics[width=\linewidth]{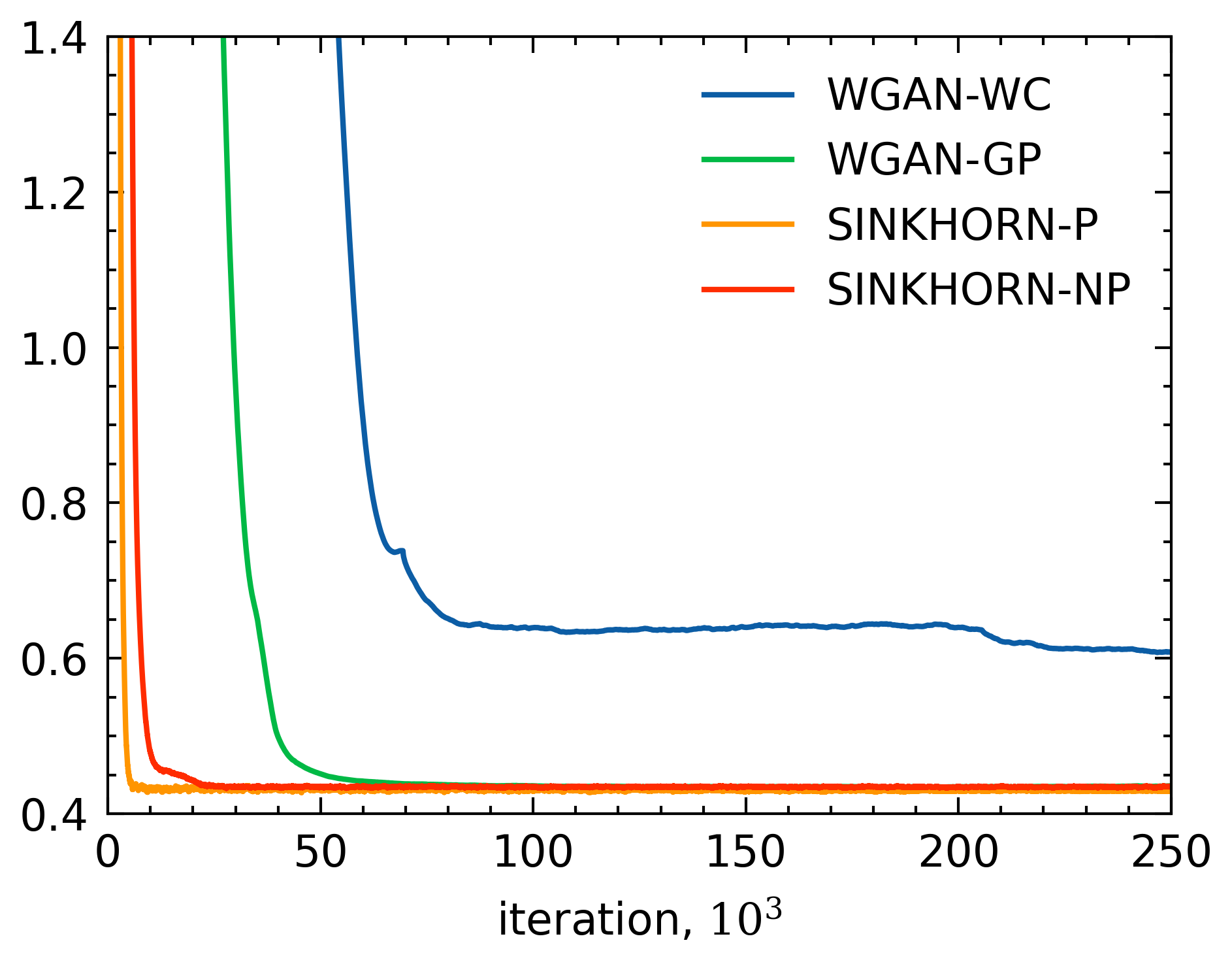}
  \end{minipage}
  \caption{Training of SGAN and WGANs for latent variable dimension $r = 4$(left) and $r = 8$(right) for a linear generator. The distance is calculated to the optimal covariance (r-PCA, top) and true covariance (bottom)} 
\end{figure*}

\section{Conclusion}
In this work we provide a comprehensive complexity analysis of entropy regularized GANs and explain their robustness. Moreover, in a specific simplified setting, the linear generator and Gaussian distributions, we derive an analytic expression for the optimal generator. This results motivates further studies on model-based designing of GANs and GANs stability.

\newpage
\acks{
This work was partly supported by a Stanford Graduate Fellowship, NSF award CCF-1704624, and the Center for Science of Information (CSoI), an NSF Science and Technology Center, under grant agreement CCF-0939370.
}

\vskip 0.2in
\bibliography{main_jmlr}

\end{document}